\documentclass[10pt]{article}

\usepackage[preprint]{tmlr}

\usepackage{amsmath,amsthm,amssymb,amsfonts}
\usepackage{mathtools}
\usepackage{bm}
\usepackage{graphicx}
\usepackage[hidelinks]{hyperref}
\usepackage{url}
\usepackage{booktabs}
\usepackage{placeins}
\usepackage{tikz}
\usetikzlibrary{arrows.meta,positioning,fit,backgrounds}

\newtheorem{theorem}{Theorem}
\newtheorem{corollary}{Corollary}
\newtheorem{lemma}{Lemma}
\newtheorem{proposition}{Proposition}
\newtheorem{definition}{Definition}
\newtheorem{assumption}{Assumption}
\newtheorem{remark}{Remark}
\newtheorem{example}{Example}

\newcommand{\R}{\mathbb{R}}
\newcommand{\E}{\mathbb{E}}

\newcommand{\Cov}{\mathrm{Cov}}
\newcommand{\tr}{\operatorname{tr}}

\newcommand{\diag}{\operatorname{diag}}
\newcommand{\op}{\mathrm{op}}
\newcommand{\SO}{\mathrm{SO}}

\newcommand{\Aset}{\mathcal{A}}
\newcommand{\Dset}{\mathcal{D}}
\newcommand{\Sset}{\mathcal{S}}

\newcommand{\inlinehead}[1]{\noindent\emph{#1}\quad}

\title{Why Do Accumulated Transformations Extrapolate?}

\author{\name Mahesh Godavarti \\\\ \addr A Carrot, Inc.}

\begin{document}

\maketitle

\begin{abstract}
PaTH Attention showed that replacing RoPE's position-indexed
rotations with accumulated data-dependent Householder reflections
yields strong length extrapolation, though performance eventually
degrades at extreme context lengths. This paper asks whether that
behavior depends on Householder-specific structure or reflects a
more general property of accumulated transformations along
source-to-query paths. We study a simpler variant that keeps RoPE's
block-diagonal $\SO(2)$ rotations but replaces position-indexed
angles with accumulated token-dependent angles. This simpler
mechanism exhibits the same qualitative pattern: improved
extrapolation followed by degradation at sufficiently long
contexts. We prove that the
result extends to accumulated orthogonal transformations satisfying
certain regularity conditions: their products become incoherent after a
finite number of steps, suppressing attention to distant tokens.
We analyze the mechanism through a stylized model of attention
that explains both behaviors.
Accumulated rotations of queries and keys create a finite mixing
window independent of context length; the per-token suppression
learned during training transfers unchanged to any evaluation length,
and high-dimensional concentration produces a score gap that suppresses
far-token attention while near-route transport preserves the target
signal. This suggests a concrete mechanism linking accumulated
transport to length extrapolation. On the other hand, a lower
bound shows that accumulated rotations must eventually degrade: as the
far set grows with context length, no choice of rotations can
guarantee preservation of the near target signal without explicit
far-mass control. Furthermore, we show that for $\SO(2)$ rotations,
also rotating values makes residual far contributions combine
incoherently, extending the extrapolation range.
Controlled transformer experiments support these predictions.
Random accumulated rotations substantially improve extrapolation
over RoPE, learned token-dependent rotations maintain
near-training-length perplexity far beyond the training context, and
rotating queries, keys, and values improves over rotating queries and
keys alone. Rotation-only
models still degrade at extreme lengths, while ALiBi remains
approximately length-stable, consistent with the need for explicit
far-mass control.
\end{abstract}

\section{Introduction}
\label{sec:intro}

PaTH Attention~\citep{yang2025path} replaces
RoPE's~\citep{su2024roformer} position-indexed rotations with
products of data-dependent Householder reflections and demonstrates
strong extrapolation at 760M parameters: each token contributes a
transformation, and the source-query relationship is determined by
the product of intervening steps rather than by absolute position.
The mechanism behind this established extrapolation phenomenon is
not well understood. This paper seeks to shed light on \emph{why}
accumulated transformations help length extrapolation.

We also ask whether the mechanism requires Householder's full
generality. We find that even RoPE's commuting block-diagonal
$\SO(2)$ rotations, with accumulated token-dependent angles instead
of position-indexed ones, extrapolate and then degrade in the same
qualitative pattern. This suggests the mechanism is not specific to
any particular orthogonal structure, and indeed we prove that
accumulated products of orthogonal transformations satisfying
certain regularity conditions become incoherent, suppressing
attention to distant tokens. We verify that both $\SO(2)$ rotations
and Householder reflections satisfy these conditions.

We call the ordered source-to-query path from token~$j$ to
query~$i$ a \emph{route}. The intervening tokens on this route
generate per-token rotations whose product gives the source-query
rotation. In RoPE, this rotation is determined entirely by the
distance $i{-}j$; in accumulated transport, it depends on which
tokens lie along the route. Distant routes pass through many
independent token-dependent steps and become incoherent, while
nearby routes can remain approximately aligned. The boundary between these two
regimes is independent of context length, so
the model sees the same near/far structure at any evaluation length.
This train/test distributional match is a necessary condition for
extrapolation, but it is not sufficient by itself.
NoPE~\citep{kazemnejad2023nope} provides an instructive
counterexample: identity transport is length-stable, but it does
not create a score gap that suppresses far tokens.
Distributional match alone does not quantify how strongly far tokens
are suppressed, whether that suppression can be overwhelmed as
context grows, or what happens to far values that survive score
selection. The analysis below addresses these questions.

We develop a stylized model of attention that explains both the
extrapolation and the eventual degradation:
\begin{enumerate}
\item \textbf{Score-side decoherence (mechanism for extrapolation).}
  Accumulated rotations of queries and keys create a finite mixing window whose
  boundary is independent of context
  length. Once training length covers this window, the near/far
  route regime is the same at training and evaluation within the
  stylized model. Within this regime, high-dimensional
  concentration produces a score gap that suppresses
  far-token attention mass, while near-route transport preserves
  the target signal. We prove that the result holds for any
  accumulated orthogonal transformation with a spectral gap
  (Appendix~\ref{app:householder-decoherence}).
\item \textbf{Far-mass lower bound (mechanism for eventual degradation).}
  Per-token suppression is stable across lengths, but the far set
  grows with context. A lower bound proves that this growth is
  fundamental: without explicit far-mass control, no choice of
  rotations can guarantee that the target signal is preserved at
  unbounded lengths. This bound applies to any orthogonal
  transport. This is consistent with the flat extrapolation of
  distance-bias methods such as
  ALiBi~\citep{press2022alibi}, which directly control the total
  far mass that rotations alone cannot bound.
\end{enumerate}
In addition, we show that for $\SO(2)$ rotations, rotating values as
well as queries and keys extends the extrapolation range, since far
values that still receive attention mass combine incoherently,
bounding the far contribution's covariance.

We do not aim to reproduce PaTH at scale. Instead, we train small
decoder-only transformers with accumulated $\SO(2)$ rotations to
isolate the mechanism; the results support each prediction. Random accumulated rotations
instantiate the independence and spectral-gap assumptions most
directly among our experimental variants and strongly improve
extrapolation over RoPE. Rotation-only models degrade gradually at
extreme lengths while ALiBi remains flat, consistent with the
far-mass requirement.
Adding value rotation further reduces long-context degradation,
consistent with the theoretical prediction. Learned token-dependent
rotations match RoPE at training length and maintain
near-training-length perplexity through $16\times$ the training
context.

The remainder of the paper follows this progression.
Section~\ref{sec:model} defines the stylized attention model
and the route notation used throughout.
Section~\ref{sec:score-control} shows that accumulated rotations
produce a content-dependent mixing window and a score gap that
forces small total far attention mass, establishing an upper bound
on far-token interference.
Section~\ref{sec:far-mass-lower-bound} proves a lower bound showing
that even stable per-token suppression cannot eliminate far-mass
leakage as the far set grows; explicit far-mass control (e.g.,
distance bias) is structurally necessary.
Section~\ref{sec:far-covariance} shows that, for the $\SO(2)$
variant, rotating values as well as queries and keys tightens the
upper bound on far-token interference by making the surviving far
contributions combine incoherently.
Section~\ref{sec:signal-gain} establishes near-signal preservation
and summarizes the combined picture.
Section~\ref{sec:empirical-context} tests the resulting predictions
in controlled $\SO(2)$ experiments.
Section~\ref{sec:discussion} discusses scope and relation to other
positional methods; Section~\ref{sec:conclusion} concludes.

\section{Preliminaries and Stylized Model}
\label{sec:model}

\subsection{Transformer attention and the aggregation model}
\label{sec:transformer-attention}

For a single attention head, let $x_j \in \R^{d_{\rm model}}$
be the residual representation at position~$j$.  Standard
scaled dot-product attention forms
\begin{equation}\label{eq:standard-qkv}
  q_i = W_Q x_i, \qquad
  k_j = W_K x_j, \qquad
  v_j = W_V x_j,
\end{equation}
then computes scores and softmax weights. With an explicit logit
scale $\lambda>0$,
\begin{equation}\label{eq:standard-softmax}
  s_{ij}
  = \frac{q_i^\top k_j}{\sqrt{d_k}} + b_{ij},
  \qquad
  \alpha_{ij}
  = \frac{\exp(\lambda s_{ij})}
  {\sum_{\ell}\exp(\lambda s_{i\ell})},
\end{equation}
where $b_{ij}$ may include a causal mask or positional bias.
The usual normalization is recovered by taking $\lambda=1$; larger
$\lambda$ makes the softmax more selective.
The attention-head output before the output projection is
\begin{equation}\label{eq:standard-attn-output}
  o_i = \sum_j \alpha_{ij}\, v_j .
\end{equation}
Thus attention has a score side, which produces the weights
$\alpha_{ij}$, and a value side, which forms the weighted
value sum.
From this point on, $d$ denotes the value/transport dimension,
assumed even; it may be the value-projection dimension rather
than the full residual dimension $d_{\rm model}$.

This section defines the score/value abstraction used to study
far interference after the score side has selected the
weights. The value-side aggregation is
\begin{equation}\label{eq:attention-aggregation}
  c_i
  = \sum_j \alpha_{ij}\, P_{j \to i}\, v_j .
\end{equation}
The ordinary transformer value sum is recovered by taking
$P_{j \to i}=I_d$, so that $c_i=o_i$. This covers both the
identity-rotation baseline and the standard RoPE-style baseline in
which Q/K uses position-dependent rotation but V is summed
directly. In the position-dependent Q/K/V comparison,
$P_{j\to i}$ is chosen from the source-query offset. In the
content-dependent Q/K/V comparison, $P_{j\to i}$ is accumulated
from the intervening tokens. The Q/K use of the same route-level
near/far split is analyzed in
Theorem~\ref{thm:score-separation}.

In the stylized model, near values carry a latent target
component; far values are background. The question is whether the
latent component remains recoverable as more far terms are added.

\subsection{Route transport}

A transport operator is an orthogonal route operator
$P_{j\to i}$ associated with the source-to-query interval.
It may be content-independent, as in position-based rotations, or
content-dependent, as in accumulated rotations generated from
intervening token representations.
On the value path, it rotates each selected value vector before
summation when value transport is enabled. On the score path, the
same route geometry can be used to compare transported query and key
features. Variants that rotate only queries and keys, including
standard RoPE-style attention, have no V-side transport:
$P_{j\to i}=I_d$.

Let $c_t$ denote the token at position~$t$.
Each position carries a per-token orthogonal step
\begin{equation}\label{eq:step-rotation}
  M_t \in O(d).
\end{equation}
For a source $j < i$, the \emph{route} from~$j$ to~$i$ is the
ordered source-to-query interval.
Define the prefix product in left-to-right (increasing-$t$) order:
$A_i = M_0 M_1 \cdots M_{i-1}$, with $A_0 = I$.
The route transport is
\begin{equation}\label{eq:route-product}
  P_{j \to i}
  = A_i^{-1} A_j
  = M_{i-1}^{-1} M_{i-2}^{-1} \cdots M_j^{-1},
\end{equation}
the product of inverse steps in query-to-source
(decreasing-$t$) order.

\emph{Content-dependent route transport} uses
\begin{equation}\label{eq:content-step}
  M_t = M(c_t),
\end{equation}
where $M(\cdot)$ maps each token representation to a learned
orthogonal matrix.
The route from $j$ to $i$ then depends on all intervening
content. All operators remain norm-preserving.

\emph{Position-only transport} uses a constant step
$M_t = M_0$ for all~$t$, giving
$P_{j \to i} = M_0^{-(i-j)}$.
The transport depends on the offset $i - j$ but not on
intervening content.

Routes to the same query are nested interval products.
For example, $P_{i-5 \to i}$ and $P_{i-2 \to i}$ share
the final two per-token steps, so far routes to a fixed query are
not independent. The covariance analysis below works with this
nested dependence directly.

A transported aggregation layer computes convex weights
$\alpha_{ij} \ge 0$, $\sum_j \alpha_{ij} = 1$, and produces
the output in~\eqref{eq:attention-aggregation}. The
score side determines the weights; when value transport is used,
the value side rotates the selected values before summation.
The block-diagonal $\SO(2)$ specialization and its implementation
convention are introduced in Section~\ref{sec:far-covariance},
where the commutative structure is used.

The following table collects the terminology and notation used in
the stylized model and proof chain.

\begin{center}
\small
\begin{tabular}{p{0.26\linewidth}p{0.66\linewidth}}
\toprule
Term or symbol & Meaning in this paper \\
\midrule
logarithms & natural logarithms \\
$\|\cdot\|_\op,\|\cdot\|_F$ & operator norm and Frobenius norm \\
$O(d),\SO(2)^{d/2}$ & orthogonal group in dimension~$d$ and
block-diagonal two-dimensional rotations ($\SO(2)$ introduced in
Section~\ref{sec:far-covariance}) \\
$\mathcal{N}(\mu,\Sigma)$ & Gaussian distribution with mean~$\mu$
and covariance~$\Sigma$ \\
score path & Q/K computation that produces attention weights \\
value path & weighted V aggregation after weights are selected \\
route & ordered source-to-query interval \\
route transport $P_{j\to i}$ & orthogonal rotation accumulated from
source~$j$ to query~$i$ \\
matched-score idealization & normalized Q/K score; the $\SO(2)$ formula
$B^{-1}\sum_b \cos\Theta_{j\to i,b}$ is given in
Section~\ref{sec:far-covariance} \\
score gap & margin between near-route and far-route Q/K scores \\
spectral gap & $\|\E[M_t]\|_{\op} < 1$: the average per-token
step is contractive \\
finite mixing window $w_{\varepsilon_{\rm mix}}$ & route length after
which accumulated content-dependent transport is decorrelated \\
active source set $\Aset_i$ & selected positions included in the
aggregation for query~$i$ \\
$\Sset_i,\Dset_i$ & target-bearing near set and far background set, with
$\Aset_i=\Sset_i\dot\cup\Dset_i$ \\
$\mathcal{N}_i(w),\mathcal{F}_i(w)$ & route-level near and far sets for
window width~$w$ \\
$\bm{\alpha}_\Sset$ & near attention-weight vector
$(\alpha_{ij})_{j\in\Sset_i}$ \\
$\bm{\alpha}_\Dset$ & far attention-weight vector
$(\alpha_{ij})_{j\in\Dset_i}$ \\
$\rho_\Dset=\|\bm{\alpha}_\Dset\|_1$ & total far mass, so
$\|\bm{\alpha}_\Sset\|_1=1-\rho_\Dset$ \\
$a_\Dset=\|\bm{\alpha}_\Dset\|_2$ & far-weight $\ell_2$ norm \\
$B_{\Sset,i}$ & transported near-signal coefficient
$\sum_{j\in\Sset_i}\alpha_{ij}P_{j\to i}$ \\
$\kappa$ & near-signal gain constant \\
$\Delta_\Dset(e)$ & far covariance
$\Cov(e_i\mid\mathcal{E}_i=e)$ \\
$\mathcal{E}_{i,L}$ & aggregation environment at query~$i$ and
length~$L$; write $\mathcal{E}_i$ when $L$ is clear, and $e$ for a
fixed realization \\
shared-background model & far-value model
$v_j=c_0G_{\rm com}+w_j$ used to isolate coherent far interference \\
\bottomrule
\end{tabular}
\end{center}

\subsection{Near and Far Route Regimes}
\label{sec:windowed-output}

Fix a query position~$i$ and consider causal source positions
$j<i$.  The route length is
$n=i-j$.
When we use route length $k$ instead of source index~$j$, the
notation $P_{i-k\to i}$ means the same source-to-query transport
as $P_{j\to i}$ with $j=i-k$.
For a window width~$w$, define the near and far
index sets
\[
  \mathcal{N}_i(w)=\{j:0<i-j<w\},
  \qquad
  \mathcal{F}_i(w)=\{j:i-j\ge w\}.
\]
These are route-level sets. They are used by the Q/K score
analysis and the value-side decomposition
\begin{equation}\label{eq:windowed-output}
  c_i = c_i^{\rm near}(w) + c_i^{\rm far}(w),
\end{equation}
where
\begin{equation}\label{eq:near-window-output}
  c_i^{\rm near}(w)
  =
  \sum_{j\in\mathcal{N}_i(w)}
  \alpha_{ij}P_{j\to i}v_j
\end{equation}
is the near-window contribution and
\begin{equation}\label{eq:far-regime-output}
  c_i^{\rm far}(w)
  =
  \sum_{j\in\mathcal{F}_i(w)}
  \alpha_{ij}P_{j\to i}v_j
\end{equation}
is the far contribution.
The near term contains the target-bearing signal.
Section~\ref{sec:score-control} gives conditions
under which content-dependent transport supplies a finite mixing
window for this split, and Section~\ref{sec:far-covariance} later
bounds the covariance of the selected far contribution.
On the score path, the same near and far route sets index the
Q/K comparisons. The transported Q/K score $S_{j\to i}$ between
source~$j$ and query~$i$ depends on the route transport
$P_{j\to i}$: near routes have transport close to the identity
and yield high scores; far routes with approximately
Haar-random transport yield low scores. Using a logit-scale
convention, the corresponding softmax weights over a finite
active source set~$\Aset_i$ are
\begin{equation}\label{eq:transported-score-softmax}
  \alpha_{ij}
  =
  \frac{\exp(\lambda S_{j\to i})}
  {\sum_{m\in\Aset_i}\exp(\lambda S_{m\to i})},
  \qquad \lambda>0 .
\end{equation}
The explicit score formula (a cosine average over block phases)
is given in Section~\ref{sec:far-covariance} when the $\SO(2)$
specialization is introduced.

\subsection{Signal-interference decomposition}
\label{sec:aggregation-env}

The active source set for query~$i$ is
$\Aset_i = \Sset_i \,\dot\cup\, \Dset_i$, where $\Sset_i$ is a
near-window target-bearing set and $\Dset_i$ is a far set.
The transported near-signal coefficient is
\begin{equation}\label{eq:signal-gain}
  B_{\Sset,i} = \sum_{j \in \Sset_i} \alpha_{ij}\,
  P_{j \to i},
\end{equation}
and the far contribution is
\begin{equation}\label{eq:interference}
  e_i = \sum_{j \in \Dset_i} \alpha_{ij}\, P_{j \to i}\, v_j.
\end{equation}
Let $\mathcal{E}_{i,L}$ denote the \emph{aggregation environment}:
the near and far index sets, the weights
$\{\alpha_{ij}^{(L)}\}$, and the route transports
$P_{j \to i}^{(L)}$. The far covariance after conditioning on
a realization $\mathcal{E}_{i,L}=e$ is
\begin{equation}\label{eq:Delta-D}
  \Delta_\Dset(e) = \Cov(e_i \mid \mathcal{E}_i = e).
\end{equation}
The analysis focuses on the setting where far values share
structure that creates nonzero cross-covariances after transport.
The canonical specialization is the \emph{shared-background model}
(Appendix~\ref{app:formal-main-results},
Definition~\ref{def:common-component}), in which far tokens share a
zero-mean Gaussian component.
Ordinary value summation leaves this component coherent; value
transport can make the weighted sum small.

\section{How Accumulated Rotations Suppress Far Attention}
\label{sec:score-control}
\label{sec:baselines}

Accumulated content-dependent orthogonal transformations create a
content-dependent mixing window. The argument proceeds in three
steps: a spectral-gap mixing result, a score gap from
high-dimensional concentration, and far-weight bounds after softmax.

The starting point is a finite mixing window
(Theorem~\ref{thm:finite-transport-window}). For any i.i.d.\
random orthogonal step matrices $M_t \in O(d)$ with
$\|\E[M_t]\|_{\mathrm{op}} \le \beta < 1$ (a \emph{spectral gap}),
the accumulated product $P_n = M_1 \cdots M_n$ satisfies
$\|\E[P_n]\|_{\mathrm{op}} \le \beta^n$. The first moment of the
route transport therefore decays geometrically, so that after a
finite number of steps $w_{\varepsilon_{\rm mix}}$ (depending on
$\beta$ and the tolerance, not on the total context length) the
accumulated transport is decorrelated.
The operator-norm condition
$\|\E[M_t]\|_{\mathrm{op}} < 1$ is a convenient sufficient
contraction condition used throughout this paper; the exact
necessary and sufficient condition for first-moment decorrelation
is $\rho(\E[M_t]) < 1$, where $\rho$ denotes the spectral radius.
(For i.i.d.\ steps, $\E[P_n] = (\E[M])^n \to 0$ if and only if
$\rho(\E[M]) < 1$.)
When $\rho(\E[M_t]) = 1$---in particular whenever $M_t$ is
deterministic---the accumulated product retains a non-decaying
component and no decorrelation occurs.
\begin{example}[$\SO(2)$: uniform step angles]
\label{ex:so2-gap}
If the step angle is uniform on $[-a,a]$ with $a>0$,
then $\beta = |\sin(a)/a| < 1$. The random-rotation experiments
(Section~\ref{sec:empirical-context}) use this distribution.
\end{example}

\begin{example}[Householder reflections]
\label{ex:householder-gap}
Let $d \ge 2$.
For Householder reflections $H_t = I - 2v_t v_t^\top$ with normal
vector $v_t$ drawn from a distribution~$\nu$ on $S^{d-1}$, the
spectral gap is
$\|\E[H_t]\|_{\mathrm{op}} = \|I - 2\Sigma_\nu\|_{\mathrm{op}}$,
where $\Sigma_\nu = \E[v_t v_t^\top]$. This is less than one
whenever the support of~$\nu$ spans~$\R^d$. Idealized PaTH-like
distributions satisfying this condition are an instance.
\end{example}

\begin{example}[Givens rotations in a random plane]
\label{ex:givens-gap}
A Givens rotation $G_{pq}(\theta)$ rotates the $(p,q)$ coordinate
plane by angle~$\theta$ and acts as the identity on all other
coordinates. If the plane $(p_t,q_t)$ is drawn uniformly from all
$\binom{d}{2}$ coordinate pairs and $\theta_t$ is uniform on
$[-a,a]$ independently, then
$\E[M_t] = \bigl(1 - \tfrac{2(1-\sin(a)/a)}{d}\bigr)\,I$, so
$\beta = \bigl|1 - \tfrac{2(1-\sin(a)/a)}{d}\bigr| < 1$
for $d \ge 2$.
\end{example}

\begin{example}[Block-diagonal $\SO(3)$ rotations]
\label{ex:so3-gap}
Assume $d = 3B$. In each 3D block, rotate by a fixed angle
$\theta \in (0,\pi)$ around an axis $\hat{n}_t$ drawn uniformly
from $S^2$. Then
$\E[M_t] = \tfrac{2\cos\theta+1}{3}\,I_3$ per block, giving
$\beta = |2\cos\theta+1|/3 < 1$.
\end{example}

\begin{example}[Torus rotation via a fixed skew-symmetric matrix]
\label{ex:full-so-gap}
Assume $d$ is even. Let $A \in \mathfrak{so}(d)$ be a fixed
nonsingular skew-symmetric matrix with eigenvalue magnitudes
$\lambda_1,\ldots,\lambda_{d/2}$, all positive. For a
content-dependent scalar
$\epsilon_t \sim \mathrm{Uniform}[-a,a]$ with $a > 0$, set
$M_t = \exp(\epsilon_t A)$. Then
$\beta = \max_j |\sin(a\lambda_j)/(a\lambda_j)| < 1$.
\end{example}

\noindent Identity and position-only transports have
$\|\E[M_t]\|_{\mathrm{op}} = 1$,
so neither RoPE nor NoPE creates this content-random mixing. Once training
length covers $w_{\varepsilon_{\rm mix}}$, longer contexts add only
routes from the already-present decorrelated regime. This
distributional match between training and evaluation is a necessary,
but not sufficient, condition for extrapolation: it ensures that the per-token
suppression the model learns during training transfers unchanged to
any evaluation length. The remainder of this section quantifies how
strong that per-token suppression is (the score gap), and
Section~\ref{sec:far-mass-lower-bound} shows why per-token
suppression alone is not enough when the far set grows.

The mixing window becomes a score gap through high-dimensional
concentration (Theorem~\ref{thm:score-separation}). Near routes
have transport close to the identity and therefore yield high
transported Q/K scores. Far routes in the decorrelated regime have
approximately Haar-random transport; sphere concentration
(Lévy's lemma on~$S^{d-1}$) then makes the probability of a large
far score decay exponentially in the dimension~$d$. Under an
additional TV-mixing condition (spectral gap of the random walk
on~$O(d)$), the product distribution converges to Haar measure on
the appropriate limiting space---$\SO(d)$ for determinant-$+1$
steps, the parity-determined determinant component for fixed
determinant-$(-1)$ steps, and full~$O(d)$ for mixed-sign steps.
Once the product is Haar-like, it sends a fixed key vector to an
approximately uniform point on the sphere, so its dot product with
a fixed query is unlikely to be large.
For the $\SO(2)$ specialization (Section~\ref{sec:far-covariance}),
the same mechanism admits a scalar analysis: each route accumulates
scalar block phases whose independence across the $d/2$ blocks lets
Hoeffding's inequality yield a score-gap rate of~$d/4$, compared
with the general Lévy rate of~$(d{-}1)/2$.
RoPE-style blocks share the same $2\!\times\!2$ rotation geometry
but have deterministic phases, so they do not satisfy the
content-random condition that drives the tail bound.

Finally, for a finite source set, the score gap plus a sufficiently
large logit scale~$\lambda$ forces small total far mass and small
far-weight $\ell_2$ norm
(Proposition~\ref{prop:score-budget}). The required scale depends
logarithmically on the far-candidate bound~$M_{\max}$; full softmax
with length-independent logits over an unbounded far set cannot
provide these bounds by itself. Formal statements and proofs for
the $\SO(2)$ case are in
Appendix~\ref{app:formal-main-results}.

Idealized Householder-step distributions satisfying the spectral-gap
conditions include PaTH-like transformations
(Example~\ref{ex:householder-gap}); verifying the TV-mixing
condition for learned PaTH transformations remains future work.
The general orthogonal formal statements and proofs are in
Appendix~\ref{app:householder-decoherence}.
Proposition~\ref{prop:householder-tv} shows that
Householder reflections whose normal vector has a smooth density
bounded below on~$S^{d-1}$ satisfy the stronger component-TV
mixing condition (Assumption~\ref{ass:householder-tv-mixing}),
and Example~\ref{ex:heat-kernel-tv} gives a heat-kernel-dithered
Householder-compatible family satisfying the same condition.

\section{Why Rotations Alone Are Not Enough}
\label{sec:far-mass-lower-bound}

Under full softmax with bounded logits, the total far attention mass
satisfies $\rho_\Dset \ge M_L/(M_L + K e^{2\lambda})$, where $K$
is the near-set size, $M_L$ is the far-set size, and $\lambda$ is
the logit bound (Proposition~\ref{prop:dense-softmax}). As
$M_L\to\infty$, this forces $\rho_\Dset\to 1$, regardless of the
rotation structure. The near-signal coefficient is then bounded by
$1-\rho_\Dset$, which vanishes at unbounded lengths
(Proposition~\ref{prop:far-mass-signal-upper}). Within this full-softmax, bounded-logit, rotation-only setting,
no choice of orthogonal transport ($\SO(2)$ blocks or Householder
products alike) can prevent this without explicit far-mass control.

This is why rotation-only models, despite the score gap from
Section~\ref{sec:score-control}, still degrade at extreme
extrapolation lengths: without an additional mechanism that directly
suppresses far attention mass, the growing far set eventually erodes
the near-signal coefficient. ALiBi's distance-dependent score bias
and FoX's data-dependent forget gates are two existing designs that
provide this control; their flat extrapolation is consistent
with this prediction.

\section{How Rotating Values Provides Additional Protection: an \texorpdfstring{$\SO(2)$}{SO(2)}-Specific Result}
\label{sec:far-covariance}
\label{sec:far-covariance-control}

Section~\ref{sec:score-control} controls how much far content is
selected by the score side. Some far mass can nevertheless remain.
If far values contain a shared background component, ordinary value
summation can preserve that component coherently even when the far
weights are small. The $\SO(2)$ Q/K/V variant targets this residual
error: by rotating values along the same accumulated route, far
values that survive score selection are made to combine
incoherently. The value-side analysis exploits the commutativity
of accumulated $\SO(2)$ rotations on the value path.

\paragraph{$\SO(2)$ notation.}
This section uses the block-diagonal $\SO(2)$ specialization of
the general orthogonal framework introduced in
Section~\ref{sec:model}. Let $B = d/2$ be the number of
two-dimensional rotation blocks. For
$\psi \in \R^{B}$, let $R(\psi)$ be the block-diagonal matrix
with a $2\!\times\!2$ rotation $R(\psi_b)$ on block~$b$. Each
position carries a block-diagonal step rotation
$R_t = R(\psi_t) \in \SO(2)^{B}$, which is a special case of the
per-token step $M_t$ from~\eqref{eq:step-rotation}. Because
block rotations commute, accumulated angles
$\Theta_i = \sum_{t<i} \psi_t$ give $A_i = R(\Theta_i)$ and
the route transport takes the simple form
$P_{j \to i} = R(\Theta_j - \Theta_i)$. Write the accumulated
phase in block~$b$ as $\Theta_{j\to i,b}$.

\emph{Content-dependent transport} uses the step-angle
\begin{equation}\label{eq:content-angle}
  \psi_t = \omega + g(c_t),
\end{equation}
where $g$ maps each token to a learned angle vector in~$\R^{B}$.

In the matched-score idealization, the block-normalized
transported Q/K score is
\begin{equation}\label{eq:block-score}
  S_{j \to i}
  = \frac{1}{B} \sum_{b=1}^{B} \cos \Theta_{j \to i,b}.
\end{equation}
Near routes have small accumulated phases in most blocks; far
routes have approximately uniform phases.

\emph{Position-only transport} uses $\psi_t = \omega$ (a
constant frequency vector), giving
$P_{j \to i} = R(\omega(j-i))$.

\paragraph{Implementation convention.}
The accumulated rotations are applied before the attention call.
Let $A_i = R(\Theta_i)$ be the prefix-product rotation at
position~$i$. For Q/K-only transport, queries are pre-rotated by
$A_i^{-1}$ and keys by~$A_j^{-1}$, so that the dot product
$q_i^\top k_j$ contains $\cos(\Theta_j - \Theta_i)$ terms as in
RoPE. (The cosine terms are invariant to the sign convention
because $\cos(\Theta_j - \Theta_i) = \cos(\Theta_i - \Theta_j)$.
The sine/antisymmetric terms do change sign, so $q^\top R(\theta)k
\neq q^\top R(-\theta)k$ for general $q,k$; however, models
trained from scratch with a fixed convention learn queries and keys
adapted to that convention, so the choice is absorbed during
training.) For Q/K/V transport, values are additionally pre-rotated by
$A_j$ before the attention call, and the attention output is
post-rotated by $A_i^{-1}$: this gives
$\sum_j \alpha_{ij}\, A_i^{-1} A_j\, v_j
= \sum_j \alpha_{ij}\, P_{j\to i}\, v_j$ as required.

\paragraph{Why commutativity matters.}
The commutativity of $\SO(2)$ blocks is what makes accumulated
route angles into sums of per-token angles. This additive structure
enables per-block Fourier analysis and the leave-one-block
decoupling used in the value-side proof below: by removing one
block's phase contribution from the score, the remaining $B-1$
blocks provide proxy weights that are independent of the removed
block's transport. The resulting adaptivity penalty is
controlled by one block's influence on the logit, of order
$e^{2\lambda/B}-1$.

\paragraph{Value-side decoherence.}
The central difficulty is that the score-selected weights depend on
the same route phases whose cancellation we want to prove. The proof
of Theorem~\ref{thm:same-path-nested-covariance} handles this by
removing one block's contribution from the score, applying a
prefix-product cancellation bound to that block with the resulting
proxy weights, and then comparing the proxy and true softmax
weights. The cost is an adaptivity penalty of order
$e^{2\lambda/B}-1$, reflecting one block's influence on the logit.
Given the score-side far-mass and far-weight $\ell_2$ bounds as
input, this leave-one-block decoupling yields a spectral covariance
bound in the shared-background model. On a high-probability
route-phase event, the far covariance satisfies
$\Delta_\Dset(e)\preceq \bar\delta^2 I_d$,
where $\bar\delta^2$ depends on the spectral gap, the
score-side far-mass bound, the logit scale, and the number of
blocks. The formal statement and proof are in
Appendix~\ref{app:formal-main-results},
Theorem~\ref{thm:same-path-nested-covariance}.

\paragraph{Householder products.}
PaTH applies accumulated Householder reflections only to queries and
keys, not to values, so its extrapolation depends entirely on
score-side mechanisms. Householder transformations could also be applied to values, but
proving value-side decoherence for noncommuting products would
require a decorrelation argument for prefix products of random
matrices, which remains open.

\section{Near-Signal Preservation and Summary}
\label{sec:signal-gain}

Sections~\ref{sec:score-control}--\ref{sec:far-covariance}
control far interference. The remaining question is whether the
near-window contribution still carries the target signal after
transport. If every target-bearing near route acts approximately as
the identity on the latent signal subspace, then their weighted
combination preserves the signal and the near-signal gain is bounded
away from zero (Lemma~\ref{thm:robust-coherence};
Appendix~\ref{app:subgaussian} gives a probabilistic sufficient
condition under sub-Gaussian angle fluctuations). Combining this
with the score-side far-mass bound yields the overall
near-signal gain condition
(Corollary~\ref{cor:score-near-gain}).

Taken together, these three mechanisms give a unified picture
within the stylized model:
score-side decoherence bounds the total far attention mass
(Section~\ref{sec:score-control}); the far-mass lower bound shows
this control is eventually overwhelmed without explicit distance bias
(Section~\ref{sec:far-mass-lower-bound}); and value-side
decoherence makes the surviving far contribution combine
incoherently (Section~\ref{sec:far-covariance}). Near-signal
preservation ensures that the useful component is not destroyed by
the same transport. Conditioned on fixed far-mass and alignment quantities, the
per-route bounds do not depend on evaluation length
(Section~\ref{sec:far-mass-lower-bound} shows why those
quantities can nevertheless degrade as the far set grows), so
the near/far structure seen during
training transfers unchanged to longer sequences.

\section{Controlled Experiments}
\label{sec:empirical-context}

\subsection{Experimental setup}

The experiments below serve as controlled demonstrations of the
analysis's predictions (score-side mixing, value-side
cancellation, the Q/K versus Q/K/V distinction, and the need for
far-mass control) in trained transformers. All models are
decoder-only causal transformers trained on OpenWebText at context
length~512 and evaluated up to 65{,}536 tokens (a 128-fold
increase). The comparison isolates two factors: whether rotations
are position-indexed or accumulated, and whether the accumulated
transport is applied only to Q/K or also to V. Full architecture,
optimizer, sampling, and evaluation details are given in
Appendix~\ref{app:experimental-details}.

\subsection{Main rotation comparison}

Table~\ref{tab:empirical-context} reports length-extrapolation
perplexity for the main rotation variants.

\begin{table}[!htbp]
\centering
\caption{Length extrapolation perplexity for the main trained
rotation variants. Models are trained at context length~512 and
evaluated up to length~65536. The ratio columns report perplexity
at the indicated length divided by perplexity at~512.}
\label{tab:empirical-context}
\resizebox{\linewidth}{!}{%
\begin{tabular}{lllccccccc}
\toprule
Model & Rule & Rotated tensors & 512 & 4096 & 8192 & 16384 & 65536 & 8K/512 & 65K/512 \\
\midrule
RoFormer/RoPE & position-indexed RoPE phase & Q/K
  & 23.54 & 223.88 & 381.82 & 600.53 & 905.18 & 16.2x & 38.5x \\
Fixed RoPE value rotation & fixed RoPE phase with value transport & Q/K/V
  & 23.17 & 127.81 & 214.18 & 306.73 & 485.66 & 9.24x & 21.0x \\
Random & accumulated random phase & Q/K
  & 24.19 & 25.90 & 38.81 & 70.63 & 181.27 & 1.60x & 7.49x \\
Random & accumulated random phase & Q/K/V
  & 24.18 & 22.58 & 24.83 & 27.54 & 38.54 & 1.03x & 1.59x \\
Learned token rotation & learned per-token phase & Q/K
  & 23.56 & 22.56 & 27.62 & 38.93 & 139.63 & 1.17x & 5.93x \\
Learned token rotation & learned per-token phase & Q/K/V
  & 23.60 & 22.82 & 27.89 & 37.58 & 102.53 & 1.18x & 4.35x \\
\bottomrule
\end{tabular}%
}
\end{table}

\subsection{Far-mass control comparison}

ALiBi tests the far-mass-control prediction: its distance-dependent
score bias directly suppresses far attention mass, providing a
reference for the rotation-only models.

\begin{table}[!htbp]
\centering
\caption{ALiBi distance-bias baseline, included to test the
far-mass-control prediction of the stylized-model analysis.}
\label{tab:alibi-reference}
\begin{tabular}{lccccccc}
\toprule
Model & 512 & 4096 & 8192 & 16384 & 65536 & 8K/512 & 65K/512 \\
\midrule
ALiBi & 23.87 & 21.75 & 22.46 & 23.19 & 22.84 & 0.94x & 0.96x \\
\bottomrule
\end{tabular}
\end{table}

\subsection{Takeaways}

The experiments support three predictions of the stylized-model analysis.
First, accumulated rotations substantially improve extrapolation
over position-indexed RoPE: learned token rotations achieve 1.17x
at 8K/512 versus RoPE's 16.2x. The random accumulated-rotation
control provides the most direct test since it satisfies the model
assumptions by construction. Second, rotating values in addition to
queries and keys consistently improves long-context behavior,
especially for random rotations (1.59x at 65K/512 versus 7.49x for
queries and keys only). Third, rotation-only models still degrade
at extreme lengths while ALiBi remains flat (0.96x at 65K/512),
consistent with the far-mass lower bound. Residual degradation
reflects both the absence of explicit far-mass control and
multi-layer dynamics beyond the single-layer stylized model.

\section{Discussion and Future Work}
\label{sec:discussion}

\subsection{What the stylized model does and does not claim}

The analysis in this paper sheds light on why accumulated
transformations extrapolate. Score-side decoherence holds for any accumulated orthogonal
transformation with a spectral gap; the $\SO(2)$ case is developed
in detail because it additionally enables value-side analysis. The theoretical results analyze the
attention mechanism in isolation: the score/value subcomputation at a
single layer with given inputs. The finite-window result does not
say that each far contribution becomes small; every rotation remains
norm-preserving. It says that route transport creates a finite
regime split whose boundary does not move with evaluated context
length. Tokens beyond the mixing window may still participate, but
they arrive from an already-present decorrelated regime rather than from a
new deterministic offset regime.

The finite-window argument rests on the per-token transformations
having enough variation across content. The general condition is
a spectral gap:
$\|\E[M_t]\|_{\mathrm{op}} < 1$, ensuring that accumulated products
mix. Examples~\ref{ex:so2-gap}--\ref{ex:full-so-gap} verify this
condition for five families of orthogonal steps, including the
$\SO(2)$ uniform-angle case used in the experiments and the
Householder case relevant to PaTH. The per-token transformation is also role-blind: it is
applied at position~$t$, independently of which query later reads
from that position. Since the same token can be relevant for one
query and irrelevant for another, the transformation cannot be
defined by setting it to the identity on signal routes and
randomizing it on far routes. Signal preservation enters through the
realized near-signal gain condition; role-blind diversity is used for
residual far suppression. The model also separates content used for
transformation generation, score-side information used for weights,
value vectors, and the latent signal; this separation is part of the
model assumption.

\paragraph{Beyond the train/test match intuition.}
Accumulated content-dependent transport creates a finite mixing
window whose boundary does not move with context length, so the
model sees the same near/far regime at training and evaluation.
However, distributional match is
a \emph{necessary} condition for extrapolation, not a sufficient
one. The score-gap analysis
(Section~\ref{sec:score-control}) shows that mixing actually
translates into small far attention mass, not just distributional
similarity. The far-mass lower bound
(Section~\ref{sec:far-mass-lower-bound}) shows that this
suppression is eventually overwhelmed, explaining the degradation
that distributional match alone does not predict. The value-side
analysis (Section~\ref{sec:far-covariance}) shows that rotating
values provides additional protection beyond score-side suppression.
Together these move from an intuition to quantitative bounds and
structural limitations.

The random-rotation experiments are useful because they instantiate
the independence and spectral-gap assumptions directly. Their
behavior therefore provides a controlled check of the stylized model,
while the learned token-rotation experiments test whether a trained
model can exploit a related mechanism.

\subsection{Relation to other position methods}

Position Interpolation~\citep{chen2024extending},
NoPE~\citep{kazemnejad2023nope}, Selective
RoPE~\citep{movahedi2026selective}, and standard RoPE all have
deterministic per-step transport, so
$\|\E[M_t]\|_{\mathrm{op}} = 1$ and the spectral-gap condition
does not apply.
Randomized RoPE~\citep{ruoss2023randomized} also extrapolates; it
randomizes position indices during training, exposing the model to
large offsets. At inference time, however, positions are
sequential and the per-step rotation is a deterministic function
of~$t$, so $\|\E[M_t]\|_{\mathrm{op}} = 1$: there is no
inference-time mixing window. The RoPE blocks are also deterministic functions
of the same position offset, so there is no cross-block
independence of the kind used in the score concentration argument.
The separate-path result in
Section~\ref{sec:separate-path-mixing} applies to an explicitly
random phase-perturbation model with independent value-side phases
after score selection, not to randomized RoPE obtained only by
sampling position indices.

Commutativity has an expressivity consequence: because the
$\SO(2)$ blocks commute, the accumulated route angle is a
\emph{sum} of the intervening step angles. For the learned
per-token variant, the route relation between source~$j$ and
query~$i$ therefore depends on the multiset (bag) of intervening
token identities, not on their order. This is strictly less
expressive than PaTH's accumulated Householder products, which
generate all of~$O(d)$ and can represent order-sensitive
transformations. Commutativity is what enables the value-side
decoherence analysis of Section~\ref{sec:far-covariance}: the
additive angle structure allows per-block Fourier analysis and
the leave-one-block decoupling that bounds far covariance.

The Forgetting
Transformer~\citep{lin2025forgetting} adds data-dependent
forget gates to the attention logits and also reports strong length
extrapolation; like ALiBi, it operates on the score computation.
Because accumulated rotations and score-side mechanisms operate on
different parts of the attention computation, they are complementary
and could in principle be combined. This is consistent with PaTH's
own strongest results, which combine accumulated Householder
transformations with the Forgetting Transformer's data-dependent
score gates (PaTH-FoX) to achieve flat extrapolation.

\paragraph{Non-normalizing activations.}
The far-mass lower bound in Section~\ref{sec:far-mass-lower-bound}
is specific to full softmax normalization. It relies on the constraint
$\sum_j \alpha_{ij}=1$: with bounded logits, a growing far set contributes
an increasingly large share of the denominator, eventually diluting the
near-token mass. A non-normalizing activation such as
$w_{ij}=\log(1+e^{s_{ij}})$ or $w_{ij}=\sigma(s_{ij})$ removes this
particular denominator effect, since each token's raw weight is computed
independently and adding far tokens does not change the raw weights of
near tokens.

However, this does not by itself solve long-context degradation. It replaces
mass dilution by a scale-control problem: many far tokens with individually
small but nonzero weights can still produce a large aggregate far contribution.
Thus the relevant quantities become the unnormalized analogues of total far
weight and far-weight $\ell_2$ norm,
$\sum_{j\in\Dset} w_{ij}$ and $(\sum_{j\in\Dset} w_{ij}^2)^{1/2}$, rather than
softmax mass. Score-side decoherence would still push far scores downward,
and value-side decoherence can still make surviving far contributions combine
incoherently, but a separate argument would be needed to show that these
unnormalized far-weight sums remain bounded as context grows. In practice,
non-normalized attention therefore requires explicit scale control, such as
normalization, learned thresholds, temperature/gating, sparsification, or other
mechanisms that make the effective far-weight tail summable.

\subsection{Limitations and future work}

The main limitation is that rotation-only score gaps do not control
total far mass under full softmax over a growing context. The
far-mass lower bound (Section~\ref{sec:far-mass-lower-bound}) makes
this structural: without explicit far-mass control (sparsity,
masking, distance bias, or a growing score gap), the near-signal
coefficient eventually degrades. The experiments confirm this:
rotation-only models degrade gradually at extreme lengths.

\paragraph{Future work.}
Several directions would extend the present analysis.
\emph{Learned noncommuting transformations:}
Appendix~\ref{app:householder-decoherence} establishes score-side
decoherence for any accumulated orthogonal transformation under
spectral-gap assumptions. A direct analysis of learned PaTH
transformations (verifying that the spectral-gap condition holds in
practice) and especially value-side decoherence for noncommuting
products, remains open.
\emph{Mechanistic probes:} direct measurements of spectral
gaps (empirical decay of $\|\E[P_n]\|_{\mathrm{op}}$ as a function
of route length), attention-mass profiles as a function of distance,
and value-side cancellation statistics in trained models would
connect the model assumptions to observed transformer behavior.
\emph{Multi-layer theory:} extending the single-layer
analysis to account for residual-stream position leakage across
layers would close the gap between the theory's length-stable
predictions and the gradual degradation observed at extreme
extrapolation.

\section{Conclusion}
\label{sec:conclusion}

PaTH Attention established that accumulated data-dependent
transformations yield strong length extrapolation; this paper
offers a mechanistic explanation. A stylized model of
attention shows that accumulated
rotations suppress far-token attention while preserving the
near-window target signal. Score-side decoherence
(Section~\ref{sec:score-control}) holds for any accumulated
orthogonal transformation with a spectral gap; idealized
Householder-step distributions, including PaTH-like transformations,
satisfy this condition. A far-mass lower bound
(Section~\ref{sec:far-mass-lower-bound}) shows this suppression is
eventually overwhelmed without explicit distance bias. In addition, we show that for $\SO(2)$ rotations, accumulated
V rotations bound the residual far contribution.
Controlled experiments support all three predictions:
accumulated rotations substantially improve extrapolation over RoPE,
rotating values improves over rotating only queries and keys,
and rotation-only models degrade at extreme lengths while ALiBi
remains flat (Section~\ref{sec:empirical-context}).

\bibliography{references}
\bibliographystyle{tmlr}

\appendix

\section{Guide to the Formal Results}
\label{app:guide}

Figure~\ref{fig:proof-dependency-graph} summarizes the dependency
structure of the formal results.
Table~\ref{tab:formal-result-roles} gives the role and downstream
use of each result.
The proof chains mirror the main-text progression:
Theorem~1 gives the $\SO(2)$ first-harmonic mixing window,
Theorem~2 converts that window into a many-block score gap, and
Proposition~1 converts the score gap into softmax far-mass and
far-weight bounds over a finite candidate set.
Lemmas~1--2 and Proposition~2 verify the stronger component-TV
mixing condition needed for general orthogonal products, including
idealized Householder reflections.
Theorem~3 records first-moment decay for arbitrary orthogonal
products, while Theorem~4 gives the stronger TV convergence
required by Theorem~5; Corollary~1 then packages the general
orthogonal score-gap and softmax-scaling result.
Propositions~3--4 prove the complementary lower bound showing that
bounded-logit full softmax eventually assigns asymptotically all
mass to the growing far regime, forcing near-signal degradation
without explicit far-mass control.
Theorem~6 proves the same-path $\SO(2)$ value-side covariance
bound, while Theorem~7 and Corollary~3 give the simpler
separate-path value-side analogue.
Finally, Lemma~3 and Corollary~2 establish near-signal
preservation when near routes remain close to identity, and
Lemma~4 gives a probabilistic sufficient condition for this
alignment.

\begin{figure}[p]
\centering
\begin{tikzpicture}[
  >=Latex,
  box/.style={
    draw, rounded corners, align=center, font=\footnotesize,
    inner sep=4pt, minimum width=26mm, minimum height=9mm
  },
  aux/.style={box, fill=gray!8},
  main/.style={box, very thick},
  resultbox/.style={box, fill=gray!15, very thick},
  dep/.style={->, thick},
  mot/.style={->, dashed, gray!70}
]

\node[aux] (ex15) at (0,0)
  {Examples 1--5\\verify spectral gap};
\node[aux] (lem12) at (5,0)
  {Lemmas 1--2\\TV-mixing criteria};
\node[aux] (ex68) at (8.5,0)
  {Examples 6--8\\TV-mixing families};

\node[main] (t1) at (0,-1.8)
  {Thm.~1\\$\SO(2)$ mixing window};
\node[aux] (p2) at (5,-1.3)
  {Prop.~2\\Householder TV mixing};
\node[aux] (ass1) at (11.5,0)
  {Assump.~1\\first-moment gap};
\node[aux] (t3) at (9.5,-1.8)
  {Thm.~3\\first-moment decay\\(conceptual)};

\node[aux] (ass2) at (5,-2.6)
  {Assump.~2\\component-TV mixing};

\node[main] (t2) at (0,-3.6)
  {Thm.~2\\$\SO(2)$ score gap};
\node[main] (t4) at (5,-4.0)
  {Thm.~4\\TV convergence};

\node[main] (p1) at (0,-5.4)
  {Prop.~1\\softmax far-mass\\and $\ell_2$ bounds};
\node[main] (t5) at (5,-5.4)
  {Thm.~5\\orthogonal score\\separation};

\node[main] (c1) at (5,-7.0)
  {Cor.~1\\orthogonal softmax bounds};

\node[resultbox] (score) at (2.5,-8.5)
  {Score-side bounds\\
  $\rho_\Dset\le\rho_\star,\;\|\bm{\alpha}_\Dset\|_2\le a_\star$};

\node[main] (p3) at (9.5,-3.6)
  {Prop.~3\\far-mass lower bound};
\node[main] (p4) at (9.5,-5.4)
  {Prop.~4\\near-signal upper bound};
\node[resultbox] (limit) at (9.5,-7.0)
  {Limitation:\\rotations alone\\eventually fail};

\node[aux] (ass3) at (0,-10.5)
  {Assump.~3\\same-path coupling};
\node[aux] (def1) at (3.5,-10.5)
  {Def.~1\\shared-background\\far values};
\node[main] (t6) at (1.5,-12.3)
  {Thm.~6\\same-path Q/K/V\\covariance};

\node[aux] (ass4) at (7.5,-10.5)
  {Assump.~4\\separate V path};
\node[main] (t7) at (7.5,-12.3)
  {Thm.~7\\separate-path\\covariance};
\node[main] (c3) at (7.5,-14.1)
  {Cor.~3\\far-weight\\covariance bound};

\node[aux] (l4) at (10.5,-10.5)
  {Lemma~4\\sub-Gaussian angles};
\node[main] (l3) at (10.5,-12.3)
  {Lemma~3\\near-route gain};
\node[main] (c2) at (10.5,-14.1)
  {Cor.~2\\near-signal\\preservation};

\node[aux] (ex9) at (-1.5,-12.3)
  {Ex.~9\\value coherence};
\node[aux] (ex10) at (4.5,-12.3)
  {Ex.~10\\adaptive selection};

\draw[dep] (ex15) -- (t1);
\draw[dep] (t1) -- (t2);
\draw[dep] (t2) -- (p1);
\draw[dep] (p1) -- (score);

\draw[dep] (lem12) -- (p2);
\draw[mot] (p2) -- (ass2);
\draw[mot] (ex68) -- (ass2);
\draw[dep] (ass2) -- (t4);
\draw[dep] (t4) -- (t5);
\draw[dep] (t5) -- (c1);
\draw[dep] (c1) -- (score);

\draw[mot] (ass1) -- (t3);
\draw[mot] (t3) -- (t4);
\draw[dep] (p1) -- (c1);

\draw[dep] (p3) -- (p4);
\draw[dep] (p4) -- (limit);

\draw[dep] (score) -- (t6);
\draw[dep] (score) -- (c3);
\draw[dep] (score) -- (c2);

\draw[dep] (ass3) -- (t6);
\draw[dep] (def1) -- (t6);

\draw[dep] (ass4) -- (t7);
\draw[dep] (def1) -- (c3);
\draw[dep] (t7) -- (c3);

\draw[dep] (l4) -- (l3);
\draw[dep] (l3) -- (c2);

\draw[mot] (ex9) -- (t6);
\draw[mot] (ex10) -- (ass3);
\draw[mot] (ex10) -- (ass4);

\end{tikzpicture}
\caption{
Dependency graph for the formal results. Solid arrows denote
formal use; dashed arrows denote motivation or conceptual support.
The top half contains the score-side chains ($\SO(2)$ on the left,
general orthogonal in the center) and the negative lower-bound
chain on the right. The bottom half contains the value-side and
near-signal preservation chains, all fed by the shared score-side
bounds. Theorem~3 is shown as a conceptual branch; Theorem~4
supplies the TV mixing used by Theorem~5.
}
\label{fig:proof-dependency-graph}
\end{figure}
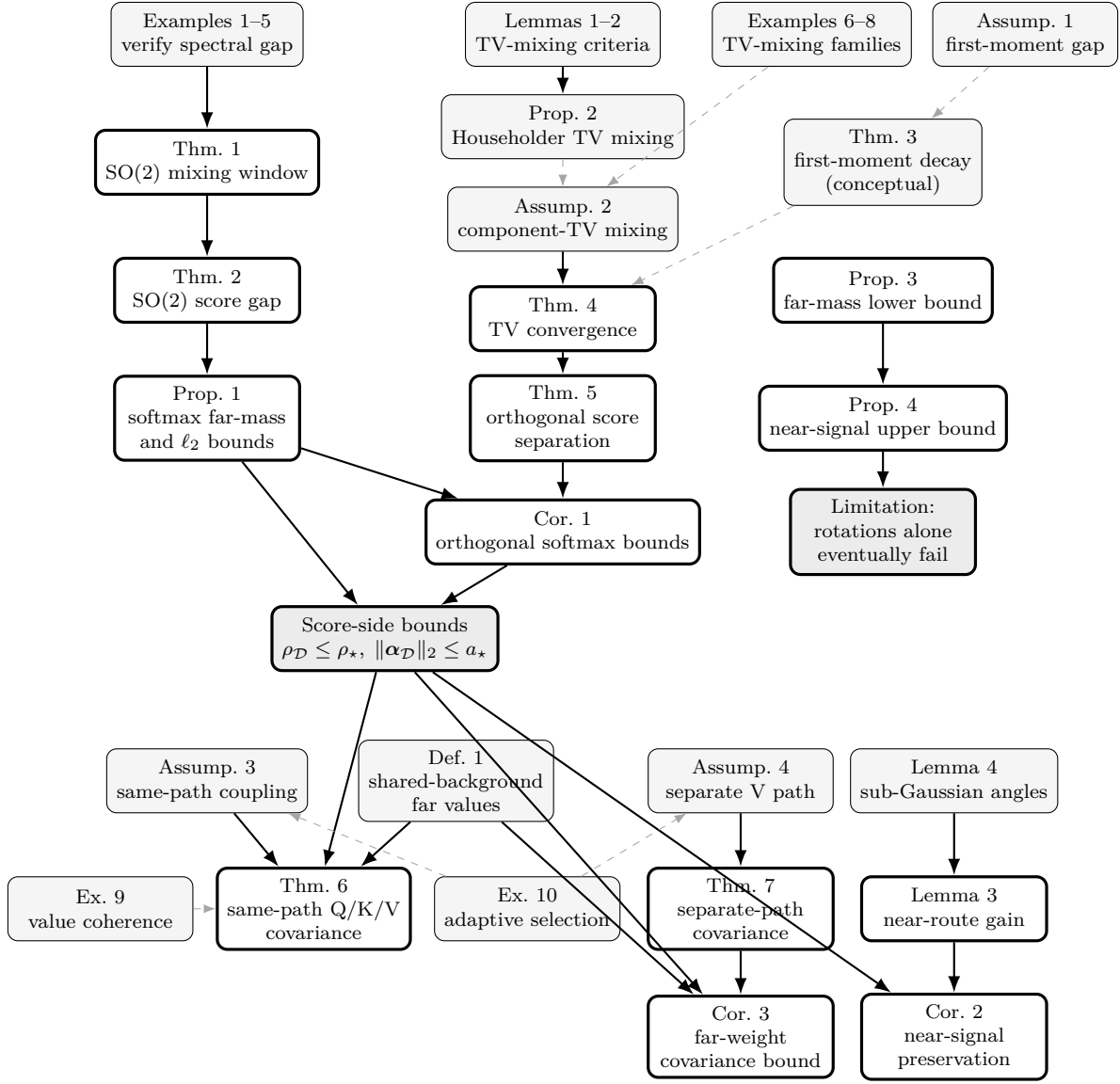

\begin{table}[p]
\centering
\caption{Role and downstream use of each formal result.}
\label{tab:formal-result-roles}
\scriptsize
\begin{tabular}{lp{48mm}p{48mm}}
\toprule
\textbf{Result} & \textbf{Purpose} & \textbf{Where used} \\
\midrule
\multicolumn{3}{l}{\emph{$\SO(2)$ score-side chain}} \\
Examples 1--5
  & Verify the first-moment/spectral-gap condition for $\SO(2)$, Householder, Givens, $\SO(3)$, and torus rotations
  & Justify the setup before Thm.~1; contrast with identity/RoPE \\
Theorem 1
  & $\SO(2)$ mixing window: accumulated first harmonic decays after a fixed number of steps
  & Supplies the far/decorrelated regime for Thm.~2 \\
Theorem 2
  & Many-block score gap via Hoeffding concentration across $\SO(2)$ blocks
  & Feeds into Prop.~1 \\
Proposition 1
  & Converts score gap into softmax far-mass and far-weight $\ell_2$ bounds
  & Bridge to Thm.~6, Cor.~2, Cor.~3; combined with Thm.~5 in Cor.~1 \\
\midrule
\multicolumn{3}{l}{\emph{General orthogonal score-side chain}} \\
Assumption 1
  & First-moment orthogonal gap $\|\E[M_t]\|_{\mathrm{op}} < 1$
  & Used by Thm.~3 \\
Assumption 2
  & Component-TV mixing: $L^2$ smoothing and spectral gap for the random walk
  & Used by Thm.~4 \\
Lemmas 1--2
  & Spread-out support criterion and analytic pushforward smoothing
  & Used by Prop.~2 and Examples 6--8 \\
Proposition 2
  & Householder reflections satisfy component-TV mixing
  & Justifies relevance to PaTH-like transformations \\
Examples 6--8
  & Heat-kernel, matrix-Fisher, and Lie-algebra families satisfy Assump.~2
  & Support plausibility of Thms.~4--5 \\
Theorem 3
  & First-moment decay for orthogonal products: $\|\E[P_n]\|_{\mathrm{op}} \le \beta^n$
  & Conceptual generalization of Thm.~1; not used by Thm.~5 \\
Theorem 4
  & TV convergence of accumulated products to Haar measure
  & Powers Thm.~5 \\
Theorem 5
  & General orthogonal score separation via TV mixing and Lévy concentration
  & Used by Cor.~1 \\
Corollary 1
  & Packages Thm.~5 + Prop.~1 into general orthogonal softmax bounds
  & Clean general-orthogonal score-side conclusion \\
\midrule
\multicolumn{3}{l}{\emph{Negative / degradation chain}} \\
Proposition 3
  & Full softmax with bounded logits assigns asymptotically all mass to a growing far set
  & Used by Prop.~4 \\
Proposition 4
  & Near-signal upper bound: far-mass leakage forces near-signal degradation
  & Formal basis for Sec.~4 \\
\midrule
\multicolumn{3}{l}{\emph{Value-side chain}} \\
Example 9
  & Position-only value rotations can remain coherent at some frequencies
  & Motivates content-dependent value analysis \\
Definition 1
  & Shared-background far-value model $v_j = c_0 G_{\mathrm{com}} + w_j$
  & Used by Thms.~6--7 and Cor.~3 \\
Assumption 3
  & Same-path route-phase score/value coupling
  & Used by Thm.~6 \\
Theorem 6
  & Same-path $\SO(2)$ far-covariance bound via leave-one-block decoupling
  & Main value-side result for Q/K/V \\
Assumption 4
  & Separate V-path: value phases independent of score-side selection
  & Used by Thm.~7 \\
Theorem 7
  & Separate-path value-covariance bound
  & Used by Cor.~3 \\
Corollary 3
  & Packages Thm.~7 with score-side bounds into a spectral covariance bound
  & Clean value-side conclusion \\
Example 10
  & Adaptive weight selection can defeat phasor cancellation
  & Justifies independence conditions in Assumps.~3--4 \\
\midrule
\multicolumn{3}{l}{\emph{Near-signal preservation chain}} \\
Lemma 3
  & Near-route closeness to identity gives near-signal gain
  & Used by Cor.~2 \\
Corollary 2
  & Combines far-mass control with near alignment for a quantitative gain bound
  & Final positive near-signal statement \\
Lemma 4
  & Sub-Gaussian short-route angles imply the closeness condition of Lemma~3
  & Probabilistic sufficient condition for Lemma~3 \\
\bottomrule
\end{tabular}
\end{table}

\section{Score-Side, Value-Side, and Signal-Preservation Proofs}
\label{app:proofs-setup}
\label{app:formal-main-results}

This section gives the formal statements and proofs for the score-side
($\SO(2)$ and general orthogonal), far-mass lower bound, value-side,
and signal-preservation results stated in the main text.

\subsection{Verification of Examples~\ref{ex:so2-gap}--\ref{ex:full-so-gap}}

The following examples verify the first-moment gap used for
finite-window decay (Theorem~\ref{thm:finite-transport-window}).

\begin{proof}[Example~\ref{ex:so2-gap}]
$|\E[e^{i\theta}]|
= \bigl|\frac{1}{2a}\int_{-a}^{a} e^{i\theta}\,d\theta\bigr|
= |\sin(a)/a| < 1$ for $a > 0$.
\end{proof}

\begin{proof}[Example~\ref{ex:householder-gap}]
$\E[H_t]
= \E[I - 2v_t v_t^\top]
= I - 2\E[v_t v_t^\top]
= I - 2\Sigma_\nu$.
The eigenvalues of $\Sigma_\nu$ lie in $[0,1]$ (it is the
second-moment matrix of unit vectors, so it is positive
semidefinite with $\tr(\Sigma_\nu) = 1$), and therefore
$\|I - 2\Sigma_\nu\|_{\mathrm{op}}
= \max_j |1 - 2\mu_j|$, where $\mu_j$ are the eigenvalues
of~$\Sigma_\nu$. This is less than one if and only if every
$\mu_j \in (0,1)$, equivalently, $\nu$ is not supported on a
proper subspace.
\end{proof}

\begin{proof}[Example~\ref{ex:givens-gap}]
For a fixed plane $(p,q)$ and $\theta \sim \mathrm{Uniform}[-a,a]$,
$\E[G_{pq}(\theta)]
= I + (\tfrac{\sin a}{a} - 1)(e_p e_p^\top + e_q e_q^\top)$,
since $\E[\cos\theta] = \sin(a)/a$ and $\E[\sin\theta] = 0$.
Averaging over all $\binom{d}{2}$ planes: each basis vector~$e_j$
appears in $d{-}1$ pairs, so
$\E_{p,q}[e_p e_p^\top + e_q e_q^\top]
= \frac{d-1}{\binom{d}{2}}\,I = \frac{2}{d}\,I$.
Therefore
$\E[M_t] = \bigl(1 - \tfrac{2(1-\sin(a)/a)}{d}\bigr)\,I$.
\end{proof}

\begin{proof}[Example~\ref{ex:so3-gap}]
By the Rodrigues formula, a rotation by angle~$\theta$ around
axis~$\hat{n}$ is
$R = \cos\theta\,I_3
+ (1-\cos\theta)\,\hat{n}\hat{n}^\top
+ \sin\theta\,[\hat{n}]_\times$.
For $\hat{n}$ uniform on $S^2$:
$\E[\hat{n}\hat{n}^\top] = \tfrac{1}{3}\,I_3$ and
$\E[[\hat{n}]_\times] = 0$.
Hence $\E[R] = \tfrac{2\cos\theta+1}{3}\,I_3$.
\end{proof}

\begin{proof}[Example~\ref{ex:full-so-gap}]
Since $d$ is even and $A$ is nonsingular, $A$ has eigenvalues
$\pm i\lambda_1,\ldots,\pm i\lambda_{d/2}$ with all
$\lambda_j > 0$. In the eigenbasis, $\exp(\epsilon A)$ is
block-diagonal with $2\times 2$ rotation blocks of angle
$\epsilon\lambda_j$.
Applying Example~\ref{ex:so2-gap} per block:
$\|\E[\exp(\epsilon A)]\|_{\mathrm{op}}
= \max_j |\sin(a\lambda_j)/(a\lambda_j)| < 1$.
\end{proof}

\subsection{Finite Transport Window}
\begin{theorem}[Finite Stable Far Regime from Content-Dependent Accumulated Rotations]
\label{thm:finite-transport-window}
\label{thm:finite-window}
Fix a query position~$i$ and consider route lengths
$n = i - j \ge 1$.
In one rotation block, define the step phasor
\[
  H_t = \exp\{-i\psi_t\},
  \qquad \psi_t = \omega + g(c_t),
\]
using the value-side sign convention for the transport definition
above.
Assume the step phasors along the route are independent and
satisfy the uniform first-harmonic bound
\[
  |\E H_t| \le \beta < 1
\]
for every position~$t$. For any tolerance
$\varepsilon_{\rm mix}\in(0,1)$, if $\beta=0$, set
$w_{\varepsilon_{\rm mix}}=1$; every route of length at least
one is already mixed. Otherwise assume $0<\beta<1$ and define
\[
  w_{\varepsilon_{\rm mix}}
  =
  \left\lceil
    \frac{\log(1/\varepsilon_{\rm mix})}{\log(1/\beta)}
  \right\rceil .
\]
Let
\[
  e^{i\Theta_n}
  = \prod_{t=i-n}^{i-1} H_t.
\]
Then every route of length $n\ge w_{\varepsilon_{\rm mix}}$ is in the
far/decorrelated first-harmonic regime:
\[
  \left|\E e^{i\Theta_n}\right|\le \varepsilon_{\rm mix} .
\]
Routes with $n<w_{\varepsilon_{\rm mix}}$ are near routes.
The boundary $w_{\varepsilon_{\rm mix}}$ depends only on the
uniform first-harmonic gap and the tolerance, not on total context
length.  The i.i.d.\ content model is a special stationary
case.
\end{theorem}

\begin{proof}
By independence of the step phasors,
\[
  \E e^{i\Theta_n}
  = \prod_{t=i-n}^{i-1} \E H_t.
\]
Taking absolute values gives
\[
  \left|\E e^{i\Theta_n}\right|
  \le \prod_{t=i-n}^{i-1} |\E H_t|
  \le \beta^n.
\]
If $\beta=0$, then $w_{\varepsilon_{\rm mix}}=1$ and
$\beta^n=0$ for every route length $n\ge1$.
If $n\ge w_{\varepsilon_{\rm mix}}$, then
$\beta^n \le \varepsilon_{\rm mix}$.
Thus every route longer than the window has first-harmonic
magnitude at most~$\varepsilon_{\rm mix}$.  Since $w_{\varepsilon_{\rm mix}}$
depends only on~$\beta$ and~$\varepsilon_{\rm mix}$, the boundary is
independent of total context length.
\end{proof}

\inlinehead{Finite-harmonic extension.}
The same argument bounds any fixed finite set of harmonics. If
$|\E H_t^m|\le \beta_m<1$ uniformly for
$m=1,\ldots,q$, then
\[
  |\E e^{im\Theta_n}|\le \beta_m^n .
\]
A single finite window bounds all harmonics
$m=1,\ldots,q$ by taking the maximum of the corresponding
windows.

\subsection{Score-Side Gap}
\begin{theorem}[Many-Block Score Gap]
\label{thm:score-separation}
Fix a query position~$i$. Suppose there is one near target
route $r \to i$ such that, for every block~$b$,
\begin{equation}\label{eq:near-phase-bound}
  |\Theta_{r \to i, b}| \le \delta,
  \qquad 0 \le \delta < \pi/2.
\end{equation}
Then
\begin{equation}\label{eq:near-score-lb}
  S_{r \to i} \ge \cos\delta.
\end{equation}
Now let $j \to i$ be a far route. Assume that the block
phases $\{\Theta_{j \to i, b}\}_{b=1}^{B}$ are independent
across~$b$, and that each block is in the
$\varepsilon_{\rm sc}$-far first-harmonic regime:
\begin{equation}\label{eq:far-block-condition}
  \bigl|\E e^{i\Theta_{j \to i, b}}\bigr| \le \varepsilon_{\rm sc}.
\end{equation}
The block phases need not be identically distributed; the
proof only uses independence across blocks and the uniform
first-harmonic bound.
Then, for every score threshold $s > \varepsilon_{\rm sc}$,
\begin{equation}\label{eq:hoeffding-score}
  \Pr\bigl[S_{j \to i} \ge s\bigr]
  \le \exp\!\biggl(-\frac{B(s - \varepsilon_{\rm sc})^2}{2}\biggr).
\end{equation}
If $\Dset_i^{\rm cand}$ is a finite candidate set of far
routes with $|\Dset_i^{\rm cand}|=M$, then
\begin{equation}\label{eq:union-score}
  \Pr\Bigl[\max_{j \in \Dset_i^{\rm cand}} S_{j \to i} \ge s\Bigr]
  \le M\exp\!\biggl(-\frac{B(s - \varepsilon_{\rm sc})^2}{2}\biggr).
\end{equation}
Consequently, if $s < \cos\delta$, then with probability at
least
\[
  1 - M\exp\!\biggl(-\frac{B(s - \varepsilon_{\rm sc})^2}{2}\biggr),
\]
the near route has a normalized score gap of at least
$\cos\delta - s$ over every route in the finite far
candidate set.
\end{theorem}

\begin{proof}
The near-route bound follows immediately from
$|\Theta_{r \to i, b}| \le \delta$, since
$\cos\Theta_{r \to i, b} \ge \cos\delta$ for every block.
Averaging over blocks gives $S_{r \to i} \ge \cos\delta$.

For a far route, define $X_b = \cos\Theta_{j \to i, b}$.
Then $X_b \in [-1, 1]$, and
condition~\eqref{eq:far-block-condition} implies
\[
  \E X_b = \mathrm{Re}\,\E e^{i\Theta_{j \to i, b}}
  \le \varepsilon_{\rm sc}.
\]
By independence across blocks,
$S_{j \to i} = \frac{1}{B}\sum_{b=1}^{B} X_b$.
Hoeffding's inequality~\cite{hoeffding1963probability} gives
\[
  \Pr\bigl[S_{j \to i} - \E S_{j \to i} \ge s - \varepsilon_{\rm sc}
  \bigr]
  \le \exp\!\biggl(-\frac{B(s - \varepsilon_{\rm sc})^2}{2}\biggr),
\]
which proves~\eqref{eq:hoeffding-score}. The
maximum-over-$M$-routes bound~\eqref{eq:union-score} follows
by the union bound. The score-gap claim follows by
combining~\eqref{eq:near-score-lb}
and~\eqref{eq:union-score}.
\end{proof}

\begin{remark}[Block independence]
\label{rem:block-independence}
Theorem~\ref{thm:score-separation} requires the block
phases $\{\Theta_{j \to i, b}\}_{b=1}^{B}$ to be independent
across blocks for each far route. This holds by construction in
the stylized random-rotation model. Learned token-dependent
rotations can create cross-block correlations through shared
content embeddings; verifying block independence (or a
suitable decorrelation substitute) for learned models is an
empirical matter.
\end{remark}

\subsection{Score Gap and Softmax Scale}
\begin{proposition}[Score Gap and Softmax Scale Give Far-Mass and Far-Weight Bounds]
\label{prop:score-budget}
\label{prop:score-scaling}
Let $\Sset$ be a near target-bearing set with
$|\Sset| = K \ge 1$, and let $\Dset$ be a far-regime candidate
set with $|\Dset| = M \le M_{\max}$ and $M_{\max}\ge1$.
Suppose the normalized
scores satisfy
$S_{j \to i} \ge c_{\rm near}$ for every $j \in \Sset$,
and
$S_{k \to i} \le c_{\rm far}$ for every $k \in \Dset$,
with score gap
$g = c_{\rm near} - c_{\rm far} > 0$.
Let the softmax logits be
$\ell_{j \to i} = \lambda\, S_{j \to i}$, $\lambda > 0$,
and let $\alpha_j$ be the resulting full-softmax weights over
$\Sset \cup \Dset$. Fix target bound levels
$\rho_\star \in (0,1)$, $a_\star > 0$.
If
\begin{equation}\label{eq:score-scale-condition}
  \lambda g
  \ge \max\!\left\{0,\;
    \log\!\left(
      \frac{M_{\max}(1 - \rho_\star)}{K \rho_\star}
    \right),\;
    \log\!\left(
      \frac{\sqrt{M_{\max}}}{K a_\star}
    \right)
  \right\},
\end{equation}
then
$\rho_\Dset := \sum_{k \in \Dset} \alpha_k \le \rho_\star$
and
$\|\bm{\alpha}_\Dset\|_2 \le a_\star$.

For the one-near-route specialization of
Theorem~\ref{thm:score-separation}, take $K = 1$,
$c_{\rm near} = \cos\delta$, and $c_{\rm far} = s$.
\end{proposition}

\begin{proof}
Let $\Delta = \lambda g$.
Since $\ell_j-\ell_k\ge\Delta$ for every
$j\in\Sset$, $k\in\Dset$, the far-to-near denominator ratio is at
most $M e^{-\Delta}/K$. Hence
\[
  \rho_\Dset
  \le \frac{M e^{-\Delta}}{K + M e^{-\Delta}}
  \le \frac{M_{\max} e^{-\Delta}}{K + M_{\max} e^{-\Delta}}.
\]
The first condition in~\eqref{eq:score-scale-condition} gives
$\rho_\Dset \le \rho_\star$. Also, for each $k\in\Dset$,
\[
  \alpha_k
  \le \frac{e^{\lambda c_{\rm far}}}
  {K e^{\lambda c_{\rm near}}}
  = \frac{e^{-\Delta}}{K}.
\]
Thus
\[
  \|\bm{\alpha}_\Dset\|_2
\le \sqrt{M}\, e^{-\Delta}/K
\le \sqrt{M_{\max}}\, e^{-\Delta}/K
\le a_\star,
\]
where the last inequality is the second condition
in~\eqref{eq:score-scale-condition}.
\end{proof}

\subsection{Score-Side Decoherence for General Orthogonal Products}
\label{app:householder-decoherence}

This subsection proves score-side decoherence for accumulated
products of i.i.d.\ random orthogonal matrices. The result applies
to any step distribution on $O(d)$ satisfying a first-moment gap
and a TV-mixing condition; idealized Householder-step
distributions, including PaTH-like transformations, satisfy these
conditions. The logical structure parallels the
$\SO(2)$ proof: spectral gap $\to$ mixing window $\to$
concentration $\to$ score gap $\to$ softmax scaling. The algebraic
machinery differs: $\SO(2)$ uses scalar Fourier analysis and
Hoeffding concentration on $B = d/2$ independent blocks, while the
general case uses random walks on $O(d)$ and Lévy concentration on
$S^{d-1}$, giving a slightly tighter concentration rate of
$(d{-}1)/2$ versus~$d/4$.

The proofs below cover three determinant cases.
If all steps have determinant~$+1$, all products lie
in~$\SO(d)$ and converge to Haar on~$\SO(d)$.
If all steps have fixed determinant~$-1$ (as in PaTH's
Householder reflections), $\det(P_n) = (-1)^n$ and the
accumulated product is confined to a single determinant component
of~$O(d)$; the TV convergence target is Haar measure on that
component.
If the step law assigns positive probability to both determinant
components (mixed sign), the products visit both components and
converge to Haar on full~$O(d)$.
In all three cases, the score-side conclusion is the same,
because every component of~$O(d)$ acts transitively on the
sphere.

\begin{assumption}[First-Moment Orthogonal Gap]
\label{ass:householder-gap}
Let $M_t \in O(d)$ be i.i.d.\ random orthogonal matrices. Define
the first-moment parameter
\[
  \beta
  \;=\; \bigl\|\E[M_t]\bigr\|_{\op}.
\]
Assume $\beta < 1$.
\end{assumption}

\begin{remark}[Householder instance]
For Householder reflections $H_t = I - 2v_t v_t^\top$ with
$v_t \sim \nu$ on $S^{d-1}$,
$\beta = \|I - 2\Sigma_\nu\|_{\op}$ where
$\Sigma_\nu = \E_\nu[vv^\top]$. The condition $\beta < 1$ holds if
and only if every eigenvalue of $\Sigma_\nu$ lies in the open
interval $(0,1)$, i.e., the support of~$\nu$ is not contained in a
proper subspace of~$\R^d$. This is the $O(d)$ analogue of the
first-harmonic gap $|\E[e^{ig(X)}]| \le \beta < 1$ used in the
$\SO(2)$ case (Theorem~\ref{thm:finite-transport-window}).
\end{remark}

\begin{remark}
\label{rem:first-moment-not-tv}
Assumption~\ref{ass:householder-gap} suffices for first-moment
decay (Theorem~\ref{thm:householder-mixing}) but not for
total-variation convergence. For example, if Householder
normals~$\nu$ are supported on the
coordinate vectors $e_1,\ldots,e_d$, then $\beta = |1 - 2/d| < 1$
for $d > 2$, yet the walk only produces diagonal sign-flip
matrices and cannot converge to Haar measure on any component
of~$O(d)$. Total-variation mixing requires the additional
condition in Assumption~\ref{ass:householder-tv-mixing}.
\end{remark}

\begin{assumption}[Component TV Mixing]
\label{ass:householder-tv-mixing}
Let $\mu$ be the law of $M_t$.
\begin{itemize}
\item[\textbf{(i)}] \textbf{Fixed determinant $+1$.}
If $M_t \in \SO(d)$ a.s., assume that $\mu$ itself satisfies
an $L^2$-smoothing and spectral-gap condition on~$\SO(d)$:
there exist an integer $r \ge 1$ and a constant
$\rho \in (0,1)$ such that $\mu^{\ast r}$ has a density
$h_r \in L^2(\SO(d))$ with respect to Haar measure
on~$\SO(d)$, and the Markov operator
$T_\mu f(x) = \int_{\SO(d)} f(xg)\, d\mu(g)$
satisfies
$\|T_\mu f\|_{L^2(\SO(d))}
\le \rho\, \|f\|_{L^2(\SO(d))}$
for every mean-zero $f \in L^2(\SO(d))$.

\item[\textbf{(ii)}] \textbf{Fixed determinant $-1$.}
If all steps share $\det M_t = -1$ (e.g., Householder
reflections), let $\mu_2 = \mu \ast \mu$ be the law of the
even-step product $M_1 M_2 \in \SO(d)$. Assume that $\mu_2$
satisfies the $L^2$-smoothing and spectral-gap condition
on~$\SO(d)$ as in~(i), with the Markov operator
\[
  T_{\mu_2} f(x)
  \;=\; \int_{\SO(d)} f(xg)\, d\mu_2(g).
\]

\item[\textbf{(iii)}] \textbf{Mixed sign.}
If $\mu$ assigns positive probability to both
$\SO(d)$ and $O^-(d)$, assume that $\mu$ satisfies
the $L^2$-smoothing and spectral-gap condition on~$O(d)$:
there exist $r \ge 1$ and $\rho \in (0,1)$ such that
$\mu^{\ast r}$ has a density
$h_r \in L^2(O(d))$ with respect to Haar measure
on~$O(d)$, and
$T_\mu f(x) = \int_{O(d)} f(xg)\, d\mu(g)$
satisfies
$\|T_\mu f\|_{L^2(O(d))}
\le \rho\, \|f\|_{L^2(O(d))}$
for every mean-zero $f \in L^2(O(d))$.
\end{itemize}
\end{assumption}

The following two lemmas are used by both the Householder
proposition and the heat-kernel example below.

\begin{lemma}[Spread-out support criterion]
\label{lem:doeblin}
Let $\eta$ be a probability measure on $\SO(d)$ with an $L^2$
Haar density: $d\eta(g) = h(g)\,dg$, $h \in L^2(\SO(d))$.
Suppose $I \in \mathrm{supp}(\eta)$ and $\mathrm{supp}(\eta)$
generates~$\SO(d)$. Then
$\|T_\eta\|_{L^2_0 \to L^2_0} < 1$,
and $\eta$ satisfies Assumption~\ref{ass:householder-tv-mixing}
with $r = 1$.

Moreover, if $h(g) \ge m > 0$ for Haar-a.e.\ $g$, the explicit
bound $\|T_\eta\|_{L^2_0 \to L^2_0} \le 1 - m$ holds.
\end{lemma}

\begin{proof}
Since $h \in L^2(\SO(d))$, the operator $T_\eta$ acts on $L^2(G)$
with kernel $K(x,y) = h(x^{-1}y)$, and
$\|K\|_{L^2(G \times G)}^2 = \|h\|_2^2 < \infty$,
so $T_\eta$ is Hilbert--Schmidt, hence compact.

Assume for contradiction that
$\|T_\eta\|_{L^2_0 \to L^2_0} = 1$.
Since $T_\eta$ is compact, its norm is attained: there exists
nonzero $f \in L^2_0(G)$ with $\|f\|_2 = 1$ and
$\|T_\eta f\|_2 = 1$. Writing $R_g f(x) = f(xg)$, each $R_g$ is
unitary and $T_\eta f = \int_G R_g f\, d\eta(g)$, so
\[
  1 = \|T_\eta f\|_2
    = \Bigl\|\int_G R_g f\, d\eta(g)\Bigr\|_2
    \le \int_G \|R_g f\|_2\, d\eta(g) = 1.
\]
Equality in the Hilbert-space triangle inequality forces
$R_g f$ to be a common vector for $\eta$-a.e.~$g$. By strong
continuity of $g \mapsto R_g f$, this extends to every
$g \in \mathrm{supp}(\eta)$. Since $I \in \mathrm{supp}(\eta)$,
$R_g f = f$ for all $g \in \mathrm{supp}(\eta)$. Because
$\mathrm{supp}(\eta)$ generates $\SO(d)$, $f$ is
right-invariant under all of $\SO(d)$, hence constant a.e.
But $f \in L^2_0$ forces $f = 0$, contradicting $\|f\|_2 = 1$.

For the explicit bound when $h \ge m > 0$: write
$\eta = m\,\mathrm{Haar} + (1-m)\,\nu'$ where
$d\nu'(g) = (h(g) - m)\,dg/(1-m)$.
For $f \in L^2_0$, Haar-invariance gives
$\int_G f(xg)\,dg = 0$, so
$T_\eta f = (1-m)\,T_{\nu'} f$, and
$\|T_\eta f\|_2 \le (1-m)\|f\|_2$.
\end{proof}

\begin{remark}[Determinant-$(-1)$ variant]
\label{rem:doeblin-det-minus}
If $\mu$ is supported on
$O^-(d) = \{Q \in O(d) : \det Q = -1\}$ with density
$h \in L^2(O^-(d),\eta^-)$ satisfying $h \ge m > 0$ a.e.\
with respect to normalized Haar measure~$\eta^-$ on
$O^-(d)$, then the even-step law $\mu_2 = \mu \ast \mu$ has
density
$h_2(g) = \int_{O^-(d)} h(x)\,h(x^{-1}g)\,d\eta^-(x)
\ge m^2$ on $\SO(d)$. By Cauchy--Schwarz,
$h_2(g) \le \|h\|_2^2$, so
$h_2 \in L^\infty(\SO(d)) \subset L^2(\SO(d))$.
Lemma~\ref{lem:doeblin} then gives
$\|T_{\mu_2}\|_{L^2_0 \to L^2_0} \le 1 - m^2 < 1$.
In particular, the conclusion holds whenever $h$ is bounded
above and below by positive constants.
\end{remark}

\begin{remark}[Mixed-sign variant]
\label{rem:doeblin-mixed}
Lemma~\ref{lem:doeblin} extends directly to~$O(d)$. If $\mu$
is a probability measure on~$O(d)$ with an $L^2$ density~$h$
with respect to Haar measure on~$O(d)$, $I \in \mathrm{supp}(\mu)$,
and $\mathrm{supp}(\mu)$ generates~$O(d)$, then
$T_\mu$ is Hilbert--Schmidt on $L^2(O(d))$ and the same
compactness/triangle-equality argument gives
$\|T_\mu\|_{L^2_0 \to L^2_0} < 1$.
In particular, if $\mu$ assigns positive probability to both
$\SO(d)$ and $O^-(d)$ and has a density bounded below by
$m > 0$ on~$O(d)$, then
$\|T_\mu\|_{L^2_0 \to L^2_0} \le 1 - m < 1$.
\end{remark}

\begin{lemma}[Analytic pushforward]
\label{lem:analytic-pushforward}
Let $M, N$ be compact real-analytic manifolds with
$\dim N = n$, and let $F \colon M \to N$ be real-analytic.
Let $a$ be a smooth density on~$M$. Assume that on every
connected component of~$M$ that intersects
$\mathrm{supp}(a)$, the differential $DF$ has rank~$n$ at
at least one point. Then the pushforward $F_*(a\,dm)$ is
absolutely continuous with respect to smooth volume on~$N$,
with density in $L^p(N)$ for some $p > 1$.
\end{lemma}

\begin{proof}
On each connected component meeting $\mathrm{supp}(a)$, at
least one $n \times n$ minor of~$DF$ is a nonzero
real-analytic function. Hence $F$ is generically a submersion
on that component, outside a proper analytic set. By
rectilinearization/monomialization of subanalytic
maps~\citep{BierstoneMilman1988}, the density of the
pushforward is locally a finite sum of monomial-type
singularities, each lying in $L^{1+\varepsilon}$ for some
$\varepsilon > 0$~\citep[see also][]{GlazerHendelSodin2022}.
Compactness of~$M$ gives a finite cover and therefore a
uniform positive~$\varepsilon$. Taking $p = 1 + \varepsilon$
proves the claim.
\end{proof}

\begin{proposition}[Householder reflections satisfy component-TV mixing]
\label{prop:householder-tv}
Let $d \ge 2$ and let $\nu$ be a probability distribution on
$S^{d-1}$ with a smooth density bounded below by a positive
constant. Let $H(v) = I - 2vv^\top$ and let $\mu$ be the law of
$H(v)$ for $v \sim \nu$. Then the Householder walk satisfies
Assumption~\ref{ass:householder-tv-mixing}.
\end{proposition}

\begin{proof}
Each $H(v)$ has determinant~$-1$, so we work with the even-step
law $\mu_2 = \mu \ast \mu$ on $G = \SO(d)$. Let
$n = \dim G = d(d{-}1)/2$.

\emph{Step~1: Generation and aperiodicity.}
By the Cartan--Dieudonn\'e theorem, every element of $O(d)$ is a
product of at most~$d$ Householder reflections, and every element
of $\SO(d)$ is a product of an even number. Since $\nu$ has full
support on~$S^{d-1}$, the semigroup generated by
$\{H(u)H(v) : u,v \in S^{d-1}\}$ is all of~$\SO(d)$.
Moreover, $I = H(v)H(v) \in \mathrm{supp}(\mu_2)$ for every~$v$, so
$\mathrm{supp}(\mu_2)$ is not contained in any coset of a proper closed
subgroup.

\emph{Step~2: $L^p$-smoothing via a full-rank product map.}
Let $N = n = d(d{-}1)/2$ and consider the product map
\[
  \Phi_N \colon (S^{d-1})^{2N} \to \SO(d), \quad
  \Phi_N(v_1,\ldots,v_{2N})
  = H(v_1)H(v_2) \cdots H(v_{2N}).
\]
Then $\mu_2^{\ast N}$ is the pushforward of $\nu^{\otimes 2N}$
under $\Phi_N$. We exhibit a point where $D\Phi_N$ is surjective.

Enumerate the coordinate pairs $1 \le i < j \le d$ as
$(i_1,j_1),\ldots,(i_N,j_N)$. For the $\ell$-th two-reflection
block, set $v_{2\ell-1} = v_{2\ell} = e_{i_\ell}$. At this base
point, every block is the identity:
$H(e_{i_\ell})H(e_{i_\ell}) = I$.
Now vary only $v_{2\ell}$ in the tangent direction
$e_{j_\ell} \in T_{e_{i_\ell}} S^{d-1}$. Using
$DH_a[\xi] = -2(\xi a^\top + a\xi^\top)$ for $\xi \perp a$, and
the fact that all other blocks remain at the identity, the
resulting tangent vector at $I \in \SO(d)$ is
\[
  H(e_{i_\ell})\, DH_{e_{i_\ell}}[e_{j_\ell}]
  = 2(e_{i_\ell} e_{j_\ell}^\top
     - e_{j_\ell} e_{i_\ell}^\top)
  = 2A_{i_\ell j_\ell},
\]
where $A_{ij} = e_i e_j^\top - e_j e_i^\top$ is the standard
basis element of $\mathfrak{so}(d)$. Ranging over all $N$ pairs
gives all $N$ basis elements, so $D\Phi_N$ is surjective at this
point.

The domain $(S^{d-1})^{2N}$ is compact and connected for
$d \ge 2$, and $\Phi_N$ is real-analytic. Since $D\Phi_N$ is
surjective at the point above,
Lemma~\ref{lem:analytic-pushforward} applies to the smooth
density of $\nu^{\otimes 2N}$. Therefore $\mu_2^{\ast N}$ has a
density $h_N \in L^p(\SO(d))$ for some $p > 1$.

\emph{Step~3: Upgrade to $L^2$.}
By Young's convolution inequality on the compact group~$G$,
if $h_N \in L^p$ with $p > 1$, then after $m$ further
self-convolutions the integrability exponent satisfies
$1/p_m = 1 - m(1 - 1/p)$ (as long as $1/p_m > 0$).
Choose $m$ so that $m(1 - 1/p) \ge 1/2$, giving $p_m \ge 2$.
Since $G$ has finite Haar measure, $L^{p_m}(G) \subseteq L^2(G)$.
Setting $r_0 = Nm$, we obtain
$\mu_2^{\ast r_0} = h_{r_0}\, dg$ with
$h_{r_0} \in L^2(\SO(d))$.

\emph{Step~4: Spectral gap.}
The measure $\mu_2^{\ast r_0}$ has an $L^2$ Haar density
(Step~3). By Step~1, $I \in \mathrm{supp}(\mu_2)$ and
$\mathrm{supp}(\mu_2)$ generates~$\SO(d)$; since
$\mathrm{supp}(\mu_2^{\ast r_0}) \supseteq \mathrm{supp}(\mu_2)$,
the same holds for $\mu_2^{\ast r_0}$.
Lemma~\ref{lem:doeblin} applied to $\eta = \mu_2^{\ast r_0}$
gives
$\|T_{\mu_2}^{r_0}\|_{L^2_0 \to L^2_0} < 1$.
Since $(H(v_1)H(v_2))^{-1} = H(v_2)H(v_1)$ and $v_1, v_2$ are
i.i.d., $\mu_2$ is symmetric, so $T_{\mu_2}$ is self-adjoint on
$L^2(G)$. For a self-adjoint operator,
$\|T^n\| = \|T\|^n$, hence
$\|T_{\mu_2}\|_{L^2_0 \to L^2_0}
= \|T_{\mu_2}^{r_0}\|_{L^2_0 \to L^2_0}^{1/r_0} < 1$,
proving Assumption~\ref{ass:householder-tv-mixing}.
\end{proof}

\begin{example}[Heat-kernel-dithered orthogonal steps]
\label{ex:heat-kernel-tv}
Let $G_{\tau,t} \sim q_\tau$ be independent draws from the heat
kernel on $\SO(d)$ at time $\tau > 0$, and let $L_t$ be any
i.i.d.\ orthogonal steps (with a fixed determinant
$\sigma \in \{+1,-1\}$), independent of~$G_{\tau,t}$.
Define $M_t = L_t\, G_{\tau,t}$.
The heat kernel $q_\tau$ is smooth and strictly positive on
compact $\SO(d)$, so
$m_\tau := \min_{g \in \SO(d)} q_\tau(g) > 0$.
If $\sigma = +1$, the one-step density satisfies
$h(g) = \int q_\tau(\ell^{-1}g)\,d\nu_L(\ell) \ge m_\tau$,
so Lemma~\ref{lem:doeblin} gives
Assumption~\ref{ass:householder-tv-mixing} with
$\rho \le 1 - m_\tau$.
If $\sigma = -1$, the one-step density on $O^-(d)$ is bounded
below by $m_\tau$, so
Remark~\ref{rem:doeblin-det-minus} gives
$\rho \le 1 - m_\tau^2$.
The first-moment gap also holds:
$\E[G_{\tau,t}] = a_\tau I$ with $0 < a_\tau < 1$ by
conjugation-invariance of~$q_\tau$, so
$\|\E[M_t]\|_{\op} \le a_\tau < 1$.

In particular, taking $L_t = H(v_t)$ (a Householder reflection)
gives a smoothed Householder example with determinant~$-1$. This
does not assert that learned PaTH steps satisfy
Assumption~\ref{ass:householder-tv-mixing}; it shows that the
assumption is nonempty in a Householder-compatible family, and
that the dither can be made arbitrarily small by taking
$\tau \to 0$.
\end{example}

\begin{example}[Matrix-Fisher law on $\SO(d)$]
\label{ex:matrix-fisher-tv}
For $A \in \R^{d \times d}$, the matrix-Fisher density on $\SO(d)$
is
$h_A(Q) = Z_A^{-1} \exp(\mathrm{tr}(A^\top Q))$,
where $Z_A = \int_{\SO(d)} \exp(\mathrm{tr}(A^\top Q))\,dQ$.
Since $h_A$ is smooth and strictly positive on compact $\SO(d)$,
it is bounded below by some $m_A > 0$, so
Lemma~\ref{lem:doeblin} applies. The crude bound
$|\mathrm{tr}(A^\top Q)| \le \|A\|_*$ (nuclear norm) gives
$h_A(Q) \ge e^{-2\|A\|_*}$, and hence,
if $A \ne 0$, Assumption~\ref{ass:householder-tv-mixing} holds
with $r = 1$ and $\rho \le 1 - e^{-2\|A\|_*} < 1$.
If $A = 0$, then $h_A \equiv 1$ is Haar measure, so $T_\mu$
maps every mean-zero function to zero; Assumption~\ref{ass:householder-tv-mixing} holds with any
$\rho \in (0,1)$.
The first-moment gap also holds. If $A = 0$, then
$\E[Q] = 0$. If $A \ne 0$, the full support of~$h_A$ implies
that $Qx$ is not almost surely constant for any unit
vector~$x$, so $\|\E[Q]x\| < 1$ by strict Jensen's inequality;
compactness of the unit sphere gives
$\|\E[Q]\|_{\op} < 1$.
\end{example}

\begin{example}[Local Lie-algebra noise]
\label{ex:lie-algebra-tv}
Let $\mathfrak{so}(d)$ be the Lie algebra of $\SO(d)$. Choose
$r > 0$ small enough that the exponential map
$\exp \colon B_{\mathfrak{so}(d)}(0,r) \to \SO(d)$
is a diffeomorphism onto its image. Let
$X_t \in \mathfrak{so}(d)$ have a smooth density supported in
$B(0,r)$ that is bounded and positive on a smaller ball
$B(0,r_0)$, and define $M_t = \exp(X_t)$.
Then the law $\mu$ of $M_t$ has an $L^2$ density with respect to
Haar measure. Its support contains the open neighborhood
$\exp(B(0,r_0))$ of the identity, so
$I \in \mathrm{supp}(\mu)$ and, since $\SO(d)$ is connected,
$\mathrm{supp}(\mu)$ generates all of $\SO(d)$.
Lemma~\ref{lem:doeblin} gives
$\|T_\mu\|_{L^2_0 \to L^2_0} < 1$.
Unlike the previous examples, the density need not be bounded
below on all of $\SO(d)$; the non-explicit part of
Lemma~\ref{lem:doeblin} provides the spectral gap.
The first-moment gap follows by the same Jensen argument as the
Matrix-Fisher case: the support of $\mu$ contains an open
neighborhood of~$I$, so $M_t x$ is not a.s.\ constant for any
unit~$x$, giving $\|\E[M_t]\|_{\op} < 1$.
\end{example}

\begin{theorem}[First-Moment Decay for Orthogonal Products]
\label{thm:householder-mixing}
Under Assumption~\ref{ass:householder-gap}, for the accumulated
product $P_n = M_1 \cdots M_n$ of $n$ i.i.d.\ steps,
\[
  \bigl\|\E[P_n]\bigr\|_{\op} \;\le\; \beta^n.
\]
For any tolerance $\varepsilon \in (0,1)$, if $\beta = 0$, set
$w_\varepsilon^{(1)} = 1$; every product of at least one step
already has zero first moment. Otherwise ($0 < \beta < 1$),
the first-moment mixing window is
$w_\varepsilon^{(1)} = \lceil \log(1/\varepsilon) /
\log(1/\beta) \rceil$.
\end{theorem}

\begin{proof}
By independence,
$\E[P_n] = \prod_{t=1}^n \E[M_t] = (\E[M])^n$.
By submultiplicativity of the operator norm,
$\|(\E[M])^n\|_{\op} \le \|\E[M]\|_{\op}^n = \beta^n$.
If $\beta = 0$, then $\beta^n = 0$ for every $n \ge 1$.
If $0 < \beta < 1$, setting $\beta^n \le \varepsilon$ and solving gives
$n \ge \log(1/\varepsilon) / \log(1/\beta)$.
\end{proof}

\begin{theorem}[Total Variation Convergence]
\label{thm:householder-tv}
Under Assumption~\ref{ass:householder-tv-mixing},
let $O^+(d) = \SO(d)$ and
$O^-(d) = \{Q \in O(d) : \det Q = -1\}$.

\begin{itemize}
\item[\textbf{(i)}] \textbf{Fixed determinant.}
If every step matrix has $\det M_t = \sigma \in \{+1,-1\}$,
write $\pi_n = \sigma^n$ and let $\mathrm{Haar}_{\pi_n}$ denote
normalized Haar measure on~$O^{\pi_n}(d)$.
Then there exist $C > 0$ and
$\beta_{\mathrm{TV}} \in (0,1)$ such that
\[
  d_{\mathrm{TV}}\!\bigl(\mathrm{law}(P_n),\;
  \mathrm{Haar}_{\pi_n}\bigr)
  \;\le\; C\, \beta_{\mathrm{TV}}^n
  \qquad \forall\, n \ge 0.
\]

\item[\textbf{(ii)}] \textbf{Mixed sign.}
If $\mu$ assigns positive probability to both $\SO(d)$ and
$O^-(d)$, let $\mathrm{Haar}_{O(d)}$ denote normalized Haar
measure on~$O(d)$. Then there exist $C > 0$ and
$\beta_{\mathrm{TV}} \in (0,1)$ such that
\[
  d_{\mathrm{TV}}\!\bigl(\mathrm{law}(P_n),\;
  \mathrm{Haar}_{O(d)}\bigr)
  \;\le\; C\, \beta_{\mathrm{TV}}^n
  \qquad \forall\, n \ge 0.
\]
\end{itemize}
In both cases, the TV mixing window
$w_\varepsilon^{\mathrm{TV}} = \lceil \log(C/\varepsilon) /
\log(1/\beta_{\mathrm{TV}}) \rceil$ is finite and independent of
context length.
\end{theorem}

\begin{proof}
\emph{Density-evolution convention.}
For a probability measure~$\nu$ on a compact group~$G$, define
the \emph{density-evolution operator}
$K_\nu f(x) = \int_G f(xg^{-1})\,d\nu(g)$.
If $\eta$ has Haar density~$h$, then $\eta \ast \nu$ has Haar
density~$K_\nu h$. Equivalently
$K_\nu = T_\nu^* = T_{\check\nu}$, where
$\check\nu(A) = \nu(A^{-1})$. Since $T_\nu^*$ and $T_\nu$ have
the same operator norm,
$\|K_\nu\|_{L^2_0 \to L^2_0}
= \|T_\nu\|_{L^2_0 \to L^2_0}$.
Thus the spectral-gap assumption on~$T_\nu$ gives the same
bound for the density-evolution operator~$K_\nu$.

\emph{Case~(i): fixed $\sigma = -1$.}
Each step has determinant~$-1$, so
$\det(P_n) = (-1)^n$ almost surely and
$P_n \in O^{\pi_n}(d)$. The only possible Haar limit is
therefore $\mathrm{Haar}_{\pi_n}$.

\emph{Even steps.}
For even $n = 2m$,
$\mathrm{law}(P_{2m}) = \mu_2^{\ast m}$ as a measure
on~$\SO(d)$.
Let $h_r = d\mu_2^{\ast r}/d\mathrm{Haar}_+$.
By the $L^2$-smoothing and spectral-gap condition
(Assumption~\ref{ass:householder-tv-mixing}(ii)), for $m \ge r$
the density of $\mu_2^{\ast m}$ with respect to
$\mathrm{Haar}_+$ satisfies
\[
  \Bigl\|\frac{d\mu_2^{\ast m}}{d\mathrm{Haar}_+}
  - 1\Bigr\|_{L^2}
  = \|K_{\mu_2}^{m-r}(h_r - 1)\|_{L^2}
  \le \rho^{m-r}\, \|h_r - 1\|_{L^2}.
\]
(The bound uses $\|K_{\mu_2}\|_{L^2_0} = \|T_{\mu_2}\|_{L^2_0}$
from the density-evolution convention above.)
By Cauchy--Schwarz,
$d_{\mathrm{TV}}(\mu_2^{\ast m},\, \mathrm{Haar}_+)
\le \tfrac{1}{2}\, \rho^{m-r}\, \|h_r - 1\|_{L^2}$.
Increasing the constant to cover $m < r$ gives
$d_{\mathrm{TV}}(\mathrm{law}(P_{2m}),\,
\mathrm{Haar}_+) \le C_2\, \rho^m$.

\emph{Odd steps.}
For odd $n = 2m+1$,
$\mathrm{law}(P_{2m+1}) = \mu \ast \mu_2^{\ast m}$.
Convolution contracts total variation:
$d_{\mathrm{TV}}(\mu \ast \mu_2^{\ast m},\;
\mu \ast \mathrm{Haar}_+)
\le d_{\mathrm{TV}}(\mu_2^{\ast m},\,
\mathrm{Haar}_+)$.
Since every element in the support of~$\mu$ has
determinant~$-1$, left multiplication of $\mathrm{Haar}_+$ by
such an element gives $\mathrm{Haar}_-$. Hence
$\mu \ast \mathrm{Haar}_+ = \mathrm{Haar}_-$, and
$d_{\mathrm{TV}}(\mathrm{law}(P_{2m+1}),\,
\mathrm{Haar}_-) \le C_2\, \rho^m$.

Taking $\beta_{\mathrm{TV}} = \sqrt{\rho}$ and
adjusting the constant gives
$d_{\mathrm{TV}}(\mathrm{law}(P_n),\,
\mathrm{Haar}_{\pi_n}) \le C\, \beta_{\mathrm{TV}}^n$
for all~$n$.

\emph{Case~(i): fixed $\sigma = +1$.}
All products lie in~$\SO(d)$. The argument above simplifies: no
parity splitting is needed, and
Assumption~\ref{ass:householder-tv-mixing}(i) gives the spectral
gap for~$T_\mu$, hence the same bound for~$K_\mu$.

\emph{Case~(ii): mixed sign.}
Assumption~\ref{ass:householder-tv-mixing}(iii) gives
the $L^2$-smoothing and spectral-gap condition for~$T_\mu$ on
all of~$O(d)$. Using the density-evolution operator~$K_\mu$
with $\|K_\mu\|_{L^2_0 \to L^2_0}
= \|T_\mu\|_{L^2_0 \to L^2_0} \le \rho$,
for $n \ge r$,
\[
  \Bigl\|\frac{d\mu^{\ast n}}{d\mathrm{Haar}_{O(d)}}
  - 1\Bigr\|_{L^2}
  = \|K_\mu^{n-r}(h_r - 1)\|_{L^2}
  \;\le\; \rho^{n-r}\, \|h_r - 1\|_{L^2},
\]
and Cauchy--Schwarz converts this to the stated TV
bound~\citep[cf.][Chapter~3]{diaconis1988group}.
\end{proof}

\begin{theorem}[Orthogonal Score Separation]
\label{thm:householder-score-separation}
Let $d \ge 2$.
Fix deterministic unit vectors $q, k \in \R^d$ with
$\|q\| = \|k\| = 1$. Equivalently, the same bounds hold
conditionally on any sigma-field independent of~$P_n$ with
respect to which $q$ and~$k$ are measurable.

\begin{itemize}
\item[\textbf{(a)}] \textbf{Near route.}
If $\|P_{j\to i} - I_d\|_{\op} \le \gamma$ with $\gamma < 1$, then
\[
  q^\top P_{j\to i}\, k
  \;\ge\; q^\top k \;-\; \gamma.
\]

\item[\textbf{(b)}] \textbf{Far route.}
For every $s \ge 0$ and every
$n \ge w_\varepsilon^{\mathrm{TV}}$ (the TV mixing window of
Theorem~\ref{thm:householder-tv}),
\[
  \Pr\!\bigl[q^\top P_n\, k \ge s\bigr]
  \;\le\; 2\exp\!\Bigl(-\frac{(d-1)\,s^2}{2}\Bigr) + \varepsilon.
\]

\item[\textbf{(c)}] \textbf{Union bound.}
For $M$ far candidates,
\[
  \Pr\!\Bigl[\max_{j}\, q^\top P_{j\to i}\, k \ge s\Bigr]
  \;\le\; 2M\exp\!\Bigl(-\frac{(d-1)\,s^2}{2}\Bigr)
  + M\varepsilon.
\]
\end{itemize}
\end{theorem}

\begin{proof}
\textbf{Part (a).}
$q^\top P_{j\to i}\, k
= q^\top k + q^\top (P_{j\to i} - I_d)\, k
\ge q^\top k - \|P_{j\to i} - I_d\|_{\op}\,\|q\|\,\|k\|
= q^\top k - \gamma$.

\textbf{Part (b).}
Fix $s \ge 0$. By Theorem~\ref{thm:householder-tv}
and the coupling characterization of total variation distance,
there exists a joint distribution $(P_n, U_n)$ where $U_n$
follows the appropriate Haar limit (Haar on $O^{\pi_n}(d)$ in
the fixed-determinant case, or Haar on $O(d)$ in the mixed-sign
case) and $\Pr[P_n \ne U_n] \le \varepsilon$ (using
$n \ge w_\varepsilon^{\mathrm{TV}}$).
In every case, $U_n k$ is uniformly distributed on~$S^{d-1}$:
$\SO(d)$, $O^-(d)$, and $O(d)$ all act transitively on the
sphere and preserve surface measure. The
function $f(x) = q^\top x$ is $1$-Lipschitz on~$S^{d-1}$ with
$\E[q^\top U_n k] = 0$. By Lévy's lemma
(concentration on the sphere;
see~\citet[Chapter~3]{ledoux2001concentration}
or~\citet[Theorem~5.1.4]{vershynin2018high}),
\[
  \Pr\!\bigl[q^\top U_n\, k \ge s\bigr]
  \;\le\; 2\exp\!\Bigl(-\frac{(d-1)\,s^2}{2}\Bigr).
\]
Combining with the coupling bound:
\begin{align*}
  \Pr\!\bigl[q^\top P_n\, k \ge s\bigr]
  &\le \Pr\!\bigl[q^\top U_n\, k \ge s\bigr]
    + \Pr[P_n \ne U_n] \\
  &\le 2\exp\!\Bigl(-\frac{(d-1)\,s^2}{2}\Bigr) + \varepsilon.
\end{align*}

\textbf{Part (c).}
Apply a union bound over $M$ far routes, each satisfying the
bound in part~(b).
\end{proof}

\begin{corollary}[Orthogonal Score Gap and Softmax Scaling]
\label{cor:householder-softmax}
Fix deterministic unit vectors $q, k \in \R^d$. Let $\Sset$ be a
near target-bearing set with $|\Sset| = K \ge 1$, and let $\Dset$
be a finite far candidate set with $|\Dset| = M \le M_{\max}$.
Assume every near route satisfies
$\|P_{j\to i} - I_d\|_{\op} \le \gamma$ for some
$\gamma < 1$, $j \in \Sset$.
Choose $s \ge 0$ and $\varepsilon \in (0,1)$ such that every far
route has length at least $w_\varepsilon^{\mathrm{TV}}$, and
define
\[
  c_{\rm near} = q^\top k - \gamma, \quad
  c_{\rm far} = s, \quad
  g = c_{\rm near} - c_{\rm far}.
\]
Assume $g > 0$. Then, with probability at least
$1 - \delta_{\rm orth}$, where
\[
  \delta_{\rm orth}
  = 2M\exp\!\Bigl(-\frac{(d-1)\,s^2}{2}\Bigr) + M\varepsilon,
\]
all near scores are at least $c_{\rm near}$ and all far scores are
at most $c_{\rm far}$.
Consequently, if the softmax logits are $\ell_j = \lambda S_j$ and
\[
  \lambda g
  \ge \max\!\left\{0,\;
    \log\!\left(
      \frac{M_{\max}(1-\rho_\star)}{K\rho_\star}
    \right),\;
    \log\!\left(
      \frac{\sqrt{M_{\max}}}{K a_\star}
    \right)
  \right\},
\]
then, on the same event,
$\rho_\Dset \le \rho_\star$ and
$\|\bm{\alpha}_\Dset\|_2 \le a_\star$.
\end{corollary}

\begin{proof}
The near-route lower bound follows from
Theorem~\ref{thm:householder-score-separation}(a), giving
$S_j \ge q^\top k - \gamma$ for $j \in \Sset$. The far-route
upper bound follows from
Theorem~\ref{thm:householder-score-separation}(c), with failure
probability~$\delta_{\rm orth}$. On the resulting event, the
score gap is at least $g = c_{\rm near} - c_{\rm far} > 0$.
Proposition~\ref{prop:score-scaling} then gives the stated
far-mass and far-weight bounds.
\end{proof}

\begin{remark}[Comparison with $\SO(2)$]
\label{rem:householder-comparison}
The concentration rate $(d{-}1)/2$ in the general orthogonal bound
(Theorem~\ref{thm:householder-score-separation}(b)) plays the role
of $B/2 = d/4$ in the $\SO(2)$ Hoeffding bound
(Theorem~\ref{thm:score-separation}). Both scale linearly with head
dimension~$d$, with the general bound slightly tighter (full
$d$-dimensional sphere concentration versus $d/2$ independent
scalar blocks). The mixing windows may differ: $\SO(2)$ requires
first-moment mixing ($w_\varepsilon^{(1)}$ steps), while the general
case requires TV mixing ($w_\varepsilon^{\mathrm{TV}}$ steps,
potentially larger by a factor depending on~$d$). The
idealized-model framing is analogous: both require independent
random steps and a first-moment spectral gap, with the general
case additionally requiring component-TV mixing
(Assumption~\ref{ass:householder-tv-mixing}). The gap between
idealized model and real PaTH is the same kind of gap as between
idealized model and real learned $\SO(2)$ rotations.
\end{remark}

\subsection{Full Softmax Lower Bound}

The bounded-logit assumption in the next proposition is a
length-independent-logit abstraction. In the score model used in
Theorem~\ref{thm:score-separation}, it follows immediately from
the normalized cosine score:
\[
  S_{j\to i}
  = \frac{1}{B}\sum_{b=1}^B \cos\Theta_{j\to i,b}
  \in [-1,1],
  \qquad
  \ell_{j\to i}=\lambda S_{j\to i}
  \in[-\lambda,\lambda].
\]
For an ordinary attention head, the same kind of bound follows
under explicit norm assumptions. If
\[
  \ell_{ij}
  = \frac{q_i^\top k_j}{\sqrt{d_k}} + b_{ij},
  \qquad
  \|q_i\|\le R_q,\quad \|k_j\|\le R_k,\quad |b_{ij}|\le B_b,
\]
then
\[
  |\ell_{ij}|
  \le \frac{R_qR_k}{\sqrt{d_k}} + B_b.
\]
Orthogonal Q/K transport preserves the query and key norms, so it
does not invalidate this bound. The proposition below is therefore
not a theorem about all trained transformer logits; it is the
full-softmax case in which the near/far logit
gap does not grow with the number of far candidates.

\begin{proposition}[Full Bounded Softmax Assigns
Asymptotically All Mass to the Far Regime]
\label{prop:dense-softmax}
\label{prop:dense-softmax-converse}
Let $\Sset$ be a near target-bearing set with $|\Sset| = K$,
and let $\Dset_L$ be a far-regime set with
$|\Dset_L| = M_L$. Suppose attention weights are computed by
full vanilla softmax over $\Sset \cup \Dset_L$:
\[
  \alpha_j
  = \frac{e^{\ell_j}}
  {\sum_{m \in \Sset \cup \Dset_L} e^{\ell_m}}.
\]
Assume that the logits are uniformly bounded:
$-\lambda \le \ell_j \le \lambda$ for all
$j \in \Sset \cup \Dset_L$, where $\lambda < \infty$ is
independent of~$L$. Define the total far-regime mass
$\rho_\Dset = \sum_{j \in \Dset_L} \alpha_j$.
Then
\[
  \rho_\Dset
  \ge \frac{M_L\, e^{-\lambda}}
  {K e^{\lambda} + M_L\, e^{-\lambda}}
  = \frac{M_L}{K e^{2\lambda} + M_L}.
\]
Consequently, if $K$ and $\lambda$ are fixed and
$M_L \to \infty$, then $\rho_\Dset \to 1$.
\end{proposition}

\begin{proof}
Let
$F = \sum_{j \in \Dset_L} e^{\ell_j}$ and
$N = \sum_{j \in \Sset} e^{\ell_j}$.
Then $F\ge M_L e^{-\lambda}$ and $N\le K e^\lambda$, so
\[
  \rho_\Dset
  = \frac{F}{F + N}
  \ge \frac{M_L\, e^{-\lambda}}
  {M_L\, e^{-\lambda} + K e^{\lambda}}
  = \frac{M_L}{M_L + K e^{2\lambda}}.
\]
As $M_L \to \infty$ with $K$ and $\lambda$ fixed,
$M_L/(M_L+K e^{2\lambda})\to1$.
\end{proof}

The same denominator effect holds under any fixed finite near/far
logit gap: if $\ell_j\le \ell_\star$ on $\Sset$ and
$\ell_k\ge \ell_\star-\Delta$ on a far set of size~$M_L$, then
\[
  \rho_\Dset \ge
  \frac{M_L e^{-\Delta}}{K+M_L e^{-\Delta}}\to 1 .
\]

\subsection{Far-Mass Near-Signal Upper Bound}
\begin{proposition}[Universal Near-Signal Upper Bound from Far-Mass Leakage]
\label{prop:far-mass-signal-upper}
Let
$\rho_\Dset = \sum_{j \in \Dset} \alpha_j$
be the total far-regime mass, so that
$\sum_{j \in \Sset} \alpha_j = 1 - \rho_\Dset$.
For arbitrary orthogonal value transports
$P_{j \to i} \in O(d)$,
\begin{equation}\label{eq:near-signal-upper}
  B_{\Sset,i}^\top B_{\Sset,i}
  \preceq (1-\rho_\Dset)^2 I_d .
\end{equation}
If, in addition, the weights are produced by full softmax with
$K$ near target-bearing candidates, $M_L$ far candidates, and
logits bounded in $[-\lambda, \lambda]$, then
\begin{equation}\label{eq:near-signal-upper-softmax}
  B_{\Sset,i}^\top B_{\Sset,i}
  \preceq
  \Bigl(\frac{K e^{2\lambda}}{K e^{2\lambda} + M_L}\Bigr)^2 I_d .
\end{equation}
For fixed $K$ and $\lambda$, the right-hand side tends to zero as
$M_L \to \infty$.
\end{proposition}

\begin{proof}
For any~$x$,
$\|B_{\Sset,i}\, x\|
\le \sum_{j \in \Sset} \alpha_j\, \|P_{j \to i}\, x\|
= (1 - \rho_\Dset)\, \|x\|$.
Hence
\[
  x^\top B_{\Sset,i}^\top B_{\Sset,i}x
  \le (1-\rho_\Dset)^2\|x\|^2,
\]
which proves~\eqref{eq:near-signal-upper}.
For full softmax with bounded logits,
Proposition~\ref{prop:dense-softmax} gives
$\rho_\Dset \ge M_L / (K e^{2\lambda} + M_L)$.
Substitution gives the second bound.
\end{proof}

\subsection{Position-Only Value Coherence}
\begin{example}[RoPE-Style Deterministic Value Transport and Arithmetic Coherence]
\label{prop:rope-upper}
Assume $d = 2B$. Consider deterministic RoPE-style value
transport with route length $k = i - j$:
\[
  P_{i-k \to i}
  = \diag\bigl(R(-\omega_1 k), \ldots, R(-\omega_B k)\bigr),
\]
where each $R(\theta)$ is the two-dimensional rotation by
angle~$\theta$. In the shared-background far-value model,
$v_j = c_0 G_{\rm com} + w_j$, with
$G_{\rm com} \sim \mathcal{N}(0, I_d)$,
$w_j \sim \mathcal{N}(0, \sigma_w^2 I_d)$, and all variables
independent, define
\[
  \bm{T}_\Dset
  = \sum_{k \in \Dset_L} \alpha_k\, P_{i-k \to i}.
\]
For each block~$b$, define the deterministic far-coherence
coefficient
\[
  q_{b,L}
  = \Bigl|\sum_{k \in \Dset_L} \alpha_k\, e^{-i\omega_b k}
  \Bigr|,
\]
and define $q_L = \min_{1 \le b \le B} q_{b,L}$.
Then
\begin{equation}\label{eq:rope-covariance-floor}
  \Delta_\Dset
  \succeq
  \bigl(c_0^2 q_L^2
  + \sigma_w^2 \|\bm{\alpha}_\Dset\|_2^2\bigr) I_d .
\end{equation}
\end{example}

\begin{proof}
In block~$b$, the far transport contribution is
$\sum_{k \in \Dset_L} \alpha_k\, R(-\omega_b k)$.
With
$\zeta_{b,L} = \sum_{k \in \Dset_L}
\alpha_k e^{-i\omega_b k}$, the $b$-th block of
$\bm{T}_\Dset$ has operator norm $q_{b,L}=|\zeta_{b,L}|$, and
$\bm{T}_\Dset \bm{T}_\Dset^\top$ has $b$-th block equal to
$q_{b,L}^2 I_2$.
Thus $\bm{T}_\Dset \bm{T}_\Dset^\top
\succeq q_L^2 I_d$.
In the shared-background model,
$\Delta_\Dset
= c_0^2\, \bm{T}_\Dset \bm{T}_\Dset^\top
+ \sigma_w^2 \|\bm{\alpha}_\Dset\|_2^2 I_d$.
Hence
$\Delta_\Dset
\succeq (c_0^2 q_L^2
+ \sigma_w^2 \|\bm{\alpha}_\Dset\|_2^2)\, I_d$.
\end{proof}

\subsection{Shared-Background Distant-Value Model}
\begin{definition}[Shared-Background Distant-Value Model]
\label{def:common-component}
This model explains why the analysis uses a spectral
covariance condition and why value transport can improve the
ordinary weighted-sum aggregation.

Suppose far-regime values have the form
\begin{equation}\label{eq:common-component}
  v_j = c_0\, G_{\rm com} + w_j, \qquad j \in \Dset,
\end{equation}
where $G_{\rm com} \sim \mathcal{N}(0, I_d)$ is a shared zero-mean
component, $w_j \sim \mathcal{N}(0, \sigma_w^2 I_d)$ are
idiosyncratic components. Conditional on the realized aggregation environment,
$G_{\rm com}$ and the $w_j$'s are independent of the realized
score variables, route phasors, weights, and transports, with the
$w_j$'s conditionally independent across~$j$. Equivalently, the
conditional covariance identity displayed below is assumed to hold.
That covariance identity is part of the shared-background model
assumption; it is not derived merely from conditioning on the
aggregation environment.
Then
\[
  e_\Dset
  = \sum_{j \in \Dset} \alpha_{ij}\, P_{j \to i}\, v_j
  = c_0 \Bigl(\sum_{j \in \Dset} \alpha_{ij}\, P_{j \to i}
  \Bigr) G_{\rm com}
  + \sum_{j \in \Dset} \alpha_{ij}\, P_{j \to i}\, w_j.
\]
Define
\begin{equation}\label{eq:Delta-common}
  \bm{T}_\Dset
  = \sum_{j \in \Dset} \alpha_{ij}\, P_{j \to i}.
\end{equation}
Conditioned on the realized transports and weights,
\[
  \Delta_\Dset
  = c_0^2\, \bm{T}_\Dset\, \bm{T}_\Dset^\top
  + \sigma_w^2\, \|\bm{\alpha}_\Dset\|_2^2\, I_d.
\]
The crude deterministic bound is
\[
  \|\bm{T}_\Dset\|_\op
  \le \sum_{j \in \Dset} \alpha_{ij} = \rho,
\]
and hence
\[
  \|\Delta_\Dset\|_\op
  \le c_0^2 \rho^2
  + \sigma_w^2\, \|\bm{\alpha}_\Dset\|_2^2.
\]
This bound does not use phase mixing; it holds for arbitrary
orthogonal transports and therefore also for identity
transport. For identity value transport, $P_{j\to i}=I_d$ for every far
token, so
\[
  \bm{T}_\Dset = \rho_\Dset\, I_d, \qquad
  \Delta_\Dset = c_0^2 \rho_\Dset^2\, I_d
  + \sigma_w^2 \|\bm{\alpha}_\Dset\|_2^2\, I_d.
\]
Thus ordinary value summation leaves the shared far component
fully coherent.

The value-side covariance theorem below gives a sharper
transport-specific bound for the actual nested interval
products generated by route transport. The price of the
shared suffix dependence is the factor $1/(1-\beta)$ in the
fluctuation term.
\end{definition}

\inlinehead{Prefix-product radius.}
To keep the following statements readable, define
\begin{equation}\label{eq:R-pp-def}
  \mathcal{R}_{\rm pp}(\mu,\rho,a;\eta)
  =
  \mu\rho
  + \frac{4a}{1-\beta}\sqrt{\log\frac{4B}{\eta}},
\end{equation}
where $\mu$ is the remaining first-harmonic mean term. In the
raw prefix-product bound $\mu=\beta^w$; after choosing the
mixing window one may use $\mu=\varepsilon_{\rm mix}$.

\subsection{Nested Q/K/V Route Phases}

\begin{assumption}[Route-Phase Score/Value Coupling]
\label{ass:route-phase-score-value-coupling}
Fix a query and a finite active route set
$\mathcal{K}=\mathcal{K}_\Sset\dot\cup\mathcal{K}_\Dset$,
where $\mathcal{K}_\Dset\subseteq\{k:k\ge w\}$ is the far
route set. The set~$\mathcal{K}$ and its partition into
$\mathcal{K}_\Sset,\mathcal{K}_\Dset$ are either
deterministic or measurable with respect to a sigma-field
independent of the route phasors
$\{H_{\ell,b}\}_{\ell,b}$ (see
Example~\ref{ex:adaptive-counterexample} for why adaptive selection
can invalidate the cancellation argument).
In each block~$b$, let
\[
  \Pi_{k,b}=\prod_{\ell=1}^{k}H_{\ell,b},
  \qquad |H_{\ell,b}|=1,
\]
and assume that the same route phasors enter the matched Q/K
score and the value rotation. The phasors
$\{H_{\ell,b}\}_{\ell,b}$ are independent over positions
and blocks and satisfy
\[
  |\E H_{\ell,b}|\le\beta<1 .
\]
The logits are
\[
  \ell_k
  =
  \frac{\lambda}{B}\sum_{b=1}^B
  \mathrm{Re}\,\Pi_{k,b},
  \qquad k\in\mathcal{K}.
\]
Let $\alpha_k$ be the softmax weights over~$\mathcal{K}$, and
define
\[
  \rho_\Dset=\sum_{k\in\mathcal{K}_\Dset}\alpha_k,
  \qquad
  \|\bm{\alpha}_\Dset\|_2
  =
  \Bigl(\sum_{k\in\mathcal{K}_\Dset}\alpha_k^2\Bigr)^{1/2}.
\]
Let $\mathcal{S}_{\rm sc}$ be the score-bound event
\[
  \mathcal{S}_{\rm sc}
  =
  \{\rho_\Dset\le\rho_\star,\;
    \|\bm{\alpha}_\Dset\|_2\le a_\star\}.
\]
Define the far value operator
\[
  \bm{T}_\Dset^{\rm same}
  =
  \sum_{k\in\mathcal{K}_\Dset}\alpha_k P_k ,
\]
where block~$b$ of $P_k$ acts as multiplication by
$\Pi_{k,b}$.
\end{assumption}

\begin{theorem}[Far-Covariance Bound for Nested Q/K/V Rotations]
\label{thm:same-path-nested-covariance}
Assume the route-phase score/value coupling of
Assumption~\ref{ass:route-phase-score-value-coupling}. For any
$\eta\in(0,1)$, define
\begin{equation}\label{eq:R-same-nested}
  R_{\rm same}
  =
  e^{2\lambda/B}
  \mathcal{R}_{\rm pp}(\beta^w,\rho_\star,a_\star;\eta)
  +(e^{2\lambda/B}-1)\rho_\star .
\end{equation}
Then there is an event $\mathcal{S}_{\rm val}$ such that
\[
  \Pr(\mathcal{S}_{\rm val})\ge1-\eta
\]
and, on
$\mathcal{S}_{\rm sc}\cap\mathcal{S}_{\rm val}$,
\[
  \|\bm{T}_\Dset^{\rm same}\|_\op\le R_{\rm same}.
\]
Consequently, in the shared-background far-value model satisfying
the conditional covariance identity in
Definition~\ref{def:common-component},
\[
  \Delta_\Dset
  =
  c_0^2\bm{T}_\Dset^{\rm same}
  (\bm{T}_\Dset^{\rm same})^\top
  +\sigma_w^2\|\bm{\alpha}_\Dset\|_2^2I_d
\]
satisfies
\begin{equation}\label{eq:delta-same-nested}
  \Delta_\Dset
  \preceq
  \bar\delta_{\rm same}^2 I_d,
  \qquad
  \bar\delta_{\rm same}^2
  =
  c_0^2R_{\rm same}^2+\sigma_w^2a_\star^2 .
\end{equation}
If $\mathcal{S}_{\rm sc}$ itself holds with probability at
least $1-\delta_{\rm sc}$, then the covariance bound holds with
probability at least $1-\delta_{\rm sc}-\eta$.
\end{theorem}

\paragraph{Proof idea.}
Section~\ref{sec:far-covariance} describes the leave-one-block
decoupling strategy. The operator norm of the block-diagonal
far-value transport equals the maximum over blockwise phasor sums.
For each block~$b$, the proxy weights (formed by removing block~$b$
from every logit) are independent of the block-$b$ value phases,
so a martingale/Azuma bound controls the block-$b$ phasor sum. A
union bound over blocks and a comparison between proxy and true
weights gives the operator-norm bound. This argument is specific to
block-diagonal rotations; extending it to noncommuting orthogonal
transports would require a different concentration approach.

\begin{proof}
Set $h=\lambda/B$. For block~$b$, define the leave-one-block
logit
\[
  \ell_k^{(-b)}
  =
  h\sum_{q\ne b}\mathrm{Re}\,\Pi_{k,q},
\]
and let $\alpha_k^{(-b)}$ be its softmax weights. The softmax
defining $\bm{\alpha}^{(-b)}$ is taken over the same full active
route set
$\mathcal{K}=\mathcal{K}_\Sset\dot\cup\mathcal{K}_\Dset$
as the original softmax; only block~$b$'s contribution to each
logit is removed. Since
$\mathrm{Re}\,\Pi_{k,b}\in[-1,1]$, both the numerator and the
denominator change by at most $e^h$ when block~$b$ is restored.
Thus, for every~$k$,
\begin{equation}\label{eq:leave-one-weight-comparison}
  e^{-2h}\alpha_k^{(-b)}
  \le \alpha_k
  \le e^{2h}\alpha_k^{(-b)} .
\end{equation}

Condition on all blocks except~$b$.  Then
$\alpha_k^{(-b)}$ is fixed, while
$\{\Pi_{k,b}:k\in\mathcal{K}_\Dset\}$ is a nested
prefix-product family. We now prove the needed one-block bound.
All probabilities in this paragraph are conditional on the
fixed external information and on the other blocks.
Suppress the block index and write
$H_\ell=H_{\ell,b}$,
$m_\ell=\E[H_\ell]$, and
$\Pi_k=\prod_{\ell=1}^k H_\ell$.  Let
\[
  u_b=\sum_{k\in\mathcal{K}_\Dset}
  \alpha_k^{(-b)}\Pi_{k,b}.
\]
For the filtration
$\mathcal{F}_\ell=\sigma(H_1,\ldots,H_\ell)$, define
$D_\ell=\Pi_\ell-m_\ell\Pi_{\ell-1}$ with $\Pi_0=1$.
Then $\E[D_\ell\mid\mathcal{F}_{\ell-1}]=0$ and
$|D_\ell|\le2$.  Iterating
$\Pi_\ell=m_\ell\Pi_{\ell-1}+D_\ell$ gives
\[
  \Pi_k
  =
  \prod_{q=1}^k m_q
  +\sum_{\ell=1}^k
  \Bigl(\prod_{q=\ell+1}^k m_q\Bigr)D_\ell .
\]
Therefore
\[
  u_b-\E u_b
  =
  \sum_{\ell\ge1} c_\ell D_\ell,\qquad
  c_\ell
  =
  \sum_{\substack{k\in\mathcal{K}_\Dset\\ k\ge\ell}}
  \alpha_k^{(-b)}
  \prod_{q=\ell+1}^k m_q .
\]
Since $|m_q|\le\beta$,
\[
  |c_\ell|
  \le
  \sum_{\substack{k\in\mathcal{K}_\Dset\\ k\ge\ell}}
  \alpha_k^{(-b)}\beta^{k-\ell}.
\]
Thus the sequence $(|c_\ell|)_\ell$ is bounded by the convolution
of $(\alpha_k^{(-b)})_k$ with the one-sided kernel
$(\beta^n)_{n\ge0}$. Young's convolution inequality gives
\[
  \Bigl(\sum_{\ell\ge1}|c_\ell|^2\Bigr)^{1/2}
  \le
  \Bigl(\sum_{k\in\mathcal{K}_\Dset}(\alpha_k^{(-b)})^2\Bigr)^{1/2}
  \sum_{n\ge0}\beta^n
  =
  \frac{a_\Dset^{(-b)}}{1-\beta},
\]
where
\[
  a_\Dset^{(-b)}
  =
  \Bigl(\sum_{k\in\mathcal{K}_\Dset}
  (\alpha_k^{(-b)})^2\Bigr)^{1/2}.
\]
The complex martingale sum
$\sum_{\ell\ge1}c_\ell D_\ell$ has real and imaginary parts
with increments bounded by $2|c_\ell|$.  Azuma--Hoeffding
\cite{azuma1967weighted,boucheron2013concentration} gives
\[
  \Pr\{|u_b-\E u_b|\ge r\}
  \le
  4\exp\!\biggl(
    -\frac{r^2(1-\beta)^2}
    {16(a_\Dset^{(-b)})^2}
  \biggr).
\]
With
$r=\frac{4a_\Dset^{(-b)}}{1-\beta}
\sqrt{\log(4B/\eta)}$,
this probability is at most $\eta/B$.  Also,
\[
  |\E u_b|
  \le
  \sum_{k\in\mathcal{K}_\Dset}
  \alpha_k^{(-b)}\beta^k
  \le
  \beta^w\rho_\Dset^{(-b)},
\]
where
\[
  \rho_\Dset^{(-b)}
  =
  \sum_{k\in\mathcal{K}_\Dset}\alpha_k^{(-b)}.
\]
Thus, with probability at least $1-\eta/B$,
\[
  |u_b|
  \le
  \mathcal{R}_{\rm pp}
  (\beta^w,\rho_\Dset^{(-b)},a_\Dset^{(-b)};\eta).
\]
Define
\[
  E_b
  =
  \left\{
  \left|\sum_{k\in\mathcal{K}_\Dset}
  \alpha_k^{(-b)}\Pi_{k,b}\right|
  \le
  \mathcal{R}_{\rm pp}
  (\beta^w,\rho_\Dset^{(-b)},a_\Dset^{(-b)};\eta)
  \right\},
\]
and let
\[
  \mathcal{S}_{\rm val}=\bigcap_{b=1}^B E_b .
\]
For this fixed block~$b$, the bound above is conditional on
all blocks $q\ne b$. Since the conditional failure probability is at
most $\eta/B$ for every realization of those conditioned variables,
the tower property gives $\Pr(E_b^c)\le\eta/B$ unconditionally. A
union bound over $b=1,\ldots,B$ gives
$\Pr(\mathcal{S}_{\rm val})\ge1-\eta$.

On $\mathcal{S}_{\rm sc}$,
\eqref{eq:leave-one-weight-comparison} gives
$\rho_\Dset^{(-b)}\le e^{2h}\rho_\star$ and
$a_\Dset^{(-b)}\le e^{2h}a_\star$.  Therefore
the leave-one-block contribution is at most
$e^{2h}\mathcal{R}_{\rm pp}(\beta^w,\rho_\star,a_\star;\eta)$.
Let
\[
  t_b=\sum_{k\in\mathcal{K}_\Dset}\alpha_k\Pi_{k,b}.
\]
The same comparison also gives
$|\alpha_k-\alpha_k^{(-b)}|\le(e^{2h}-1)\alpha_k$, and hence
\[
  |t_b-u_b|
  \le \sum_{k\in\mathcal{K}_\Dset}
  |\alpha_k-\alpha_k^{(-b)}|
  \le (e^{2h}-1)\rho_\Dset
  \le (e^{2h}-1)\rho_\star .
\]
Combining the bounds yields $|t_b|\le R_{\rm same}$ for every
block~$b$. Since the value operator is block diagonal,
$\|\bm{T}_\Dset^{\rm same}\|_\op=\max_b|t_b|$.
The covariance bound follows from
$\bm{T}_\Dset^{\rm same}(\bm{T}_\Dset^{\rm same})^\top
\preceq R_{\rm same}^2I_d$ and
$\|\bm{\alpha}_\Dset\|_2\le a_\star$ on
$\mathcal{S}_{\rm sc}$. The final probability statement is the
union bound with the score-bound event.
\end{proof}

\subsection{Near-Signal Gain}
\begin{lemma}[Near-Route Closeness Gives Near-Signal Gain]
\label{thm:robust-coherence}
\label{lem:near-gain}
Let $\Pi_U$ denote the orthogonal projector onto
$\mathrm{col}(U)$, and let
$\rho_\Dset = \sum_{j \in \Dset_i} \alpha_{ij}$, so that
$\sum_{j \in \Sset_i} \alpha_{ij} = 1 - \rho_\Dset$.
Suppose that, for every
$j \in \Sset_i$,
\begin{equation}\label{eq:near-rotation-error}
  \|(P_{j \to i} - I_d)\, \Pi_U\|_\op
  \le \gamma, \qquad 0 \le \gamma < 1.
\end{equation}
Then, for all $a \in \R^r$,
\begin{equation}\label{eq:robust-gain}
  \|B_{\Sset,i}\, U a\|
  \ge (1 - \rho_\Dset)(1 - \gamma)\, \|U a\|.
\end{equation}
Consequently,
\[
  U^\top B_{\Sset,i}^\top\, B_{\Sset,i}\, U
  \succeq (1 - \rho_\Dset)^2 (1 - \gamma)^2\, U^\top U.
\]
\end{lemma}

\begin{proof}
For each $j \in \Sset_i$,
\[
  P_{j \to i}\, U a = U a + (P_{j \to i} - I_d)\, U a,
\]
and
$\|(P_{j \to i} - I_d)\, U a\| \le \gamma\, \|U a\|$.
Therefore
\[
  B_{\Sset,i}\, U a
  = \sum_{j \in \Sset_i} \alpha_{ij}\, P_{j \to i}\, U a
  = (1 - \rho_\Dset)\, U a
  + \sum_{j \in \Sset_i} \alpha_{ij}\,
  (P_{j \to i} - I_d)\, U a.
\]
The error term has norm at most
$(1 - \rho_\Dset)\, \gamma\, \|U a\|$. The triangle inequality
gives
\[
  \|B_{\Sset,i}\, U a\|
  \ge (1 - \rho_\Dset)(1 - \gamma)\, \|U a\|.
\]
Squaring both sides gives the stated Loewner bound.
\end{proof}

\subsection{Score Bounds plus Near Alignment}
\begin{corollary}[Score Far-Mass Bound plus Near Alignment Gives Near-Signal Gain]
\label{cor:score-near-gain}
\label{cor:near-gain-score-budget}
Suppose the score side supplies
$\rho_\Dset \le \rho_\star$
for some $\rho_\star \in (0,1)$.
Suppose also that, for every target-bearing near token
$j \in \Sset_i$,
\[
  \|(P_{j \to i} - I_d)\, \Pi_U\|_\op
  \le \gamma, \qquad 0 \le \gamma < 1.
\]
Then the direct near-signal gain condition holds with
$\kappa_\star = (1 - \rho_\star)^2 (1 - \gamma)^2$.
That is,
\[
  U^\top B_{\Sset,i}^\top\, B_{\Sset,i}\, U
  \succeq (1 - \rho_\star)^2 (1 - \gamma)^2\, U^\top U.
\]
\end{corollary}

\begin{proof}
By Lemma~\ref{thm:robust-coherence},
\[
  U^\top B_{\Sset,i}^\top\, B_{\Sset,i}\, U
  \succeq (1 - \rho_\Dset)^2 (1 - \gamma)^2\, U^\top U.
\]
Since $\rho_\Dset \le \rho_\star$,
$1 - \rho_\Dset \ge 1 - \rho_\star$.
Therefore
$(1 - \rho_\Dset)^2 (1 - \gamma)^2
\ge (1 - \rho_\star)^2 (1 - \gamma)^2$.
\end{proof}

\section{Experimental Details}
\label{app:experimental-details}

Table~\ref{tab:hyperparams} summarizes the architecture, training,
and evaluation setup. Training data is a continuous token stream;
attention crosses document boundaries. For random-rotation models,
step angles are resampled at each forward pass; the deterministic
seed ensures that evaluation batches are identical across models but
the random angles differ across iterations.

\begin{table}[h]
\centering
\caption{Architecture, training, and evaluation hyperparameters.}
\label{tab:hyperparams}
\small
\begin{tabular}{ll}
\toprule
\multicolumn{2}{l}{\textbf{Architecture}} \\
\midrule
Type & Decoder-only causal transformer \\
Layers / width / heads / head dim & 16 / 768 / 8 / 96 \\
Normalization / activation & Pre-LayerNorm / GELU \\
Initialization & Kaiming-uniform (linear), default (embeddings) \\
Weight tying & None \\
Absolute position embeddings & None \\
\midrule
\multicolumn{2}{l}{\textbf{Training}} \\
\midrule
Dataset / vocabulary & OpenWebText / BPE 32K \\
Context length & 512 \\
Optimizer & AdamW ($\beta_1{=}0.9$, $\beta_2{=}0.999$, $\varepsilon{=}10^{-8}$) \\
Weight decay & 0.1 \\
Batch size & 32 (16{,}384 tokens/iter) \\
Steps / tokens & 200K / ${\sim}3.3 \times 10^9$ \\
Learning rate & $5{\times}10^{-4}$ (100K) $\to$ $2{\times}10^{-4}$ (50K) $\to$ $5{\times}10^{-5}$ (50K) \\
Warmup & None \\
Dropout & 0.1 (attention + residual) \\
Gradient clipping & L2 norm $\le 1.0$ \\
\midrule
\multicolumn{2}{l}{\textbf{Evaluation}} \\
\midrule
Lengths & 512 to 65{,}536 \\
Batches per length & 200 iterations, batch size 4, seed $42 + L$ \\
Metric & Cross-entropy (all positions, no burn-in) \\
\midrule
\multicolumn{2}{l}{\textbf{Hardware}} \\
\midrule
GPU / precision / seed & NVIDIA A100 80\,GB / BF16 / single seed (42) \\
\bottomrule
\end{tabular}
\end{table}

\paragraph{Model variants.}
RoFormer/RoPE uses the standard position-indexed rotary map on Q/K
only. The fixed RoPE Q/K/V baseline uses the same RoPE angles but
also applies the position-indexed rotation to values before
aggregation; it tests whether value transport alone solves
extrapolation when the phase rule is still position-indexed. Random
variants use accumulated random angle increments and instantiate
the spectral-gap assumptions most directly, testing whether
incoherent accumulated phase can protect long-context evaluation. Specifically, each step angle is drawn independently per
position and per dimension from $\mathrm{Uniform}(-f_b, f_b)$,
where $f_b = 1/10000^{2b/d}$ is the RoPE log-spaced frequency for
block~$b$. Angles are sampled independently across layers and
resampled at each forward pass (not fixed at initialization). The
same angle vector is shared across all heads within a layer.
Learned token-rotation variants use one learned per-token angle
embedding per layer; angles depend on token identity and are
accumulated along the sequence, so the source-query relation is
composed from the intervening tokens rather than from absolute
position. Variants that also rotate values test the value-side
transport predicted by the signal-interference decomposition.

\section{Additional Supporting Results}
\label{app:additional-support}

\subsection{Separate-Path Conditional Mixing Assumption}
\label{sec:separate-path-mixing}

This subsection records the separate-path random-phase case. It
applies when the far weights have already been chosen by the score
path and, conditional on that score-side information, the value-side
step phases still have independent first-harmonic mixing. For
example, it applies to a construction in which each value-side block
receives independent random step phases that are not reused by the
Q/K score path. It should not be read as a statement about
randomized RoPE obtained only by sampling position indices: in that
case all RoPE frequency blocks are functions of the same sampled
source-query offset, so the block-independence condition below is
not supplied by the positional randomization alone.

\begin{remark}[Extension to general orthogonal products]
\label{rem:separate-path-general}
The same-path value-side result
(Theorem~\ref{thm:same-path-nested-covariance}) requires $\SO(2)$
commutativity because the attention weights depend on the same
route product that transports the values; the leave-one-block-out
decoupling exploits the block-diagonal structure. Under the
separate-path assumption, this obstruction disappears: conditional
on the score-side sigma-field~$\mathcal{G}_{i,L}$, the weights are
fixed and the value-side steps are independent of them. The
value-side decoherence argument therefore extends in principle to
general orthogonal step matrices $M_t^V \in O(d)$.
The conditional first-moment gap
$\|\E[M_t^V \mid \mathcal{G}_{i,L}]\|_{\mathrm{op}} \le \beta < 1$
gives first-moment decay of the accumulated value-side product.
To obtain Lévy-type concentration and incoherent combination of
far values, one additionally needs a conditional version of the
component-TV mixing assumption
(Assumption~\ref{ass:householder-tv-mixing}) for the value-side
step law given~$\mathcal{G}_{i,L}$; the first-moment gap alone
does not imply TV convergence (cf.\ the counterexample in
Remark~\ref{rem:first-moment-not-tv}). The concrete bounds below are stated for
$\SO(2)$ blocks, where the analysis reduces to scalar phasors per
block.
\end{remark}

\begin{assumption}[Conditional V-Path First-Harmonic Gap After
Score Selection]
\label{ass:independent-far-transports}
Fix a query position~$i$. Let $\mathcal{G}_{i,L}$ be the
sigma-field generated by the score-side information used to
select the attention weights for the evaluated context
length~$L$. The far weights
$\alpha_k=\alpha_{i,i-k}$ are assumed to be
$\mathcal{G}_{i,L}$-measurable.

In each rotation block~$b$, define the value-side step phasor
\[
  H_{\ell,b}^V = e^{-i\psi_{\ell,b}^V},
\]
where $\psi_{\ell,b}^V$ is the $b$-th coordinate of the
value-side step angle
$\psi_\ell^V=\omega+g(c_\ell)$ from~\eqref{eq:content-angle}.
Conditional on $\mathcal{G}_{i,L}$, the value-side step phasors
$\{H_{\ell,b}^V\}_\ell$ are independent over positions~$\ell$,
and their conditional means
\[
  m_{\ell,b}
  = \E[H_{\ell,b}^V\mid\mathcal{G}_{i,L}]
\]
satisfy
\[
  |m_{\ell,b}|\le \beta<1
\]
for every $\ell$ and~$b$.

No independence is assumed among the route products themselves;
the route products are nested interval products and share suffix
rotations. The assumption rules out choosing the far weights by
observing the same value-side phase realizations whose
cancellation is later claimed.
\end{assumption}

\subsection{Auxiliary Separate-Path Prefix-Product Bound}
\begin{theorem}[Nested Prefix-Product Value-Covariance Bound for Score-Selected Weights]
\label{thm:value-covariance}
\label{thm:prefix-product}
Fix a query position~$i$, a far window~$w$, and a finite
evaluated far set of route lengths
$\mathcal{K}_\Dset \subseteq \{k : k \ge w\}$.
Let $\alpha_k \ge 0$ be the score-selected far weights.
In this route-length notation, $P_{i-k\to i}^V$ denotes the
value-side transport from source position $i-k$ to query
position~$i$.
Write
\[
  \rho = \sum_{k \in \mathcal{K}_\Dset} \alpha_k,
  \qquad
  a = \Bigl(\sum_{k \in \mathcal{K}_\Dset}
  \alpha_k^2\Bigr)^{1/2}.
\]
In block~$b$, define
\[
  \Pi_{k,b}^V = \prod_{\ell=1}^{k} H_{i-\ell,b}^V,
\]
and
\[
  t_b = \sum_{k \in \mathcal{K}_\Dset} \alpha_k\, \Pi_{k,b}^V.
\]
Under
Assumption~\ref{ass:independent-far-transports}, for every
$\eta \in (0,1)$, with conditional probability at least
$1 - \eta$ given $\mathcal{G}_{i,L}$,
\begin{equation}\label{eq:value-cov-bound}
  \max_{1 \le b \le B} |t_b|
  \le \mathcal{R}_{\rm pp}(\beta^w,\rho,a;\eta).
\end{equation}
Equivalently, for the block-diagonal weighted far rotation sum,
\[
  \bm{T}_\Dset^V
  = \sum_{k \in \mathcal{K}_\Dset} \alpha_k\,
  P_{i-k \to i}^V,
\]
one has, with conditional probability at least $1 - \eta$,
\begin{equation}\label{eq:transport-op-bound}
  \|\bm{T}_\Dset^V\|_\op
  \le
  \mathcal{R}_{\rm pp}
  (\beta^w,\rho,\|\bm{\alpha}_\Dset\|_2;\eta).
\end{equation}
In the shared-background far-value model satisfying the conditional
covariance identity of Definition~\ref{def:common-component},
\[
  \Delta_\Dset
  = c_0^2\, \bm{T}_\Dset^V (\bm{T}_\Dset^V)^\top
  + \sigma_w^2\, \|\bm{\alpha}_\Dset\|_2^2\, I_d,
\]
and therefore, on the same event,
\begin{equation}\label{eq:Delta-value-cov}
  \Delta_\Dset
  \preceq \biggl[
    c_0^2
    \mathcal{R}_{\rm pp}
    (\beta^w,\rho,\|\bm{\alpha}_\Dset\|_2;\eta)^2
    + \sigma_w^2\, \|\bm{\alpha}_\Dset\|_2^2
  \biggr] I_d.
\end{equation}
\end{theorem}

\begin{proof}
Condition on $\mathcal{G}_{i,L}$, so the weights are
deterministic and the value-side phasors satisfy ordinary
independence and mean bounds.  For each block, the one-block
martingale calculation in the proof of
Theorem~\ref{thm:same-path-nested-covariance}, applied with
fixed coefficients $\alpha_k$, gives
\[
  |t_b|
  \le \beta^w\rho
  +\frac{4a}{1-\beta}\sqrt{\log\frac{4B}{\eta}}
\]
with conditional failure probability at most $\eta/B$.  A union
bound over the $B$ blocks gives
\eqref{eq:value-cov-bound}. In each
two-dimensional block the matrix
$\sum_{k \in \mathcal{K}_\Dset} \alpha_k P_{i-k \to i}^V$
acts as multiplication by $t_b$; hence
$\|\bm{T}_\Dset^V\|_\op = \max_{1 \le b \le B} |t_b|$.
The covariance bound follows from
$\bm{T}_\Dset^V (\bm{T}_\Dset^V)^\top
\preceq \|\bm{T}_\Dset^V\|_\op^2 I_d$.
\end{proof}

\subsection{Adaptive Selection Counterexample}
\begin{example}[Diffuse Weights Alone Do Not Ensure Cancellation]
\label{ex:adaptive-counterexample}
Let $\zeta_1, \ldots, \zeta_M$ be independent phasors uniformly
distributed on the unit circle. Fix
$\varepsilon \in (0, \pi/2]$, and define the selected index set
\[
  A_\varepsilon = \{k : |\arg \zeta_k| \le \varepsilon\},
\]
where $\arg \zeta_k \in [-\pi, \pi)$. If
$A_\varepsilon \neq \emptyset$, define adaptive weights
\[
  \alpha_k =
  \begin{cases}
    |A_\varepsilon|^{-1}, & k \in A_\varepsilon, \\
    0, & k \notin A_\varepsilon.
  \end{cases}
\]
Then, with probability at least
$1 - \exp(-M\varepsilon / (8\pi))$,
one has
\[
  \|\bm{\alpha}\|_2 \le \sqrt{\frac{2\pi}{M\varepsilon}},
\]
and
\[
  \Bigl|\sum_{k=1}^{M} \alpha_k \zeta_k\Bigr| \ge \cos\varepsilon.
\]
Thus $\|\bm{\alpha}\|_2$ can tend to zero while the weighted
phasor sum remains bounded away from zero.
\end{example}

\begin{proof}
For each~$k$, the event $|\arg \zeta_k| \le \varepsilon$ has
probability $p = \varepsilon/\pi$.
Hence $N_\varepsilon = |A_\varepsilon|$ has the binomial
distribution $\mathrm{Bin}(M, p)$ and
$\E N_\varepsilon = Mp = M\varepsilon/\pi$.
By the standard multiplicative Chernoff bound
(see, e.g., \cite{boucheron2013concentration}),
\[
  \Pr\bigl(N_\varepsilon < \tfrac{1}{2} Mp\bigr)
  \le \exp(-Mp/8)
  = \exp(-M\varepsilon/(8\pi)).
\]
On the complementary event,
$N_\varepsilon \ge M\varepsilon/(2\pi)$.
The weights are uniform on $A_\varepsilon$, so
$\|\bm{\alpha}\|_2^2 = 1/N_\varepsilon$ and
$\|\bm{\alpha}\|_2 \le \sqrt{2\pi/(M\varepsilon)}$.
For every $k \in A_\varepsilon$,
$\mathrm{Re}\, \zeta_k = \cos(\arg \zeta_k) \ge \cos\varepsilon$.
Therefore
\[
  \mathrm{Re}\Bigl(\sum_{k=1}^{M} \alpha_k \zeta_k\Bigr)
  = \frac{1}{N_\varepsilon}
  \sum_{k \in A_\varepsilon} \mathrm{Re}\, \zeta_k
  \ge \cos\varepsilon.
\]
Since the magnitude of a complex number is at least its real
part when the real part is at least zero,
$|\sum_k \alpha_k \zeta_k| \ge \cos\varepsilon$.
\end{proof}

\subsection{Separate-Path Far-Weight to Covariance Corollary}
\begin{corollary}[Far-Mass and Far-Weight Bounds Give Spectral Far-Covariance Bound]
\label{cor:score-to-value}
\label{cor:prefix-covariance}
Let $w = w_{\varepsilon_{\rm mix}}$, so that
$\beta^w \le \varepsilon_{\rm mix}$.  Suppose the far weights
satisfy
\[
  \rho_\Dset \le \rho_\star, \qquad
  \|\bm{\alpha}_\Dset\|_2 \le a_\star.
\]
For each finite evaluated far set, the bound has no explicit
dependence on the raw number of far routes; only total far
mass and far-weight $\ell_2$ norm enter.
Under the shared-background far-value model satisfying the
conditional covariance identity of Definition~\ref{def:common-component},
and under Assumption~\ref{ass:independent-far-transports},
Theorem~\ref{thm:value-covariance} implies that, with
probability at least $1 - \eta$,
\[
  \Delta_\Dset \preceq \bar\delta_{\rm pp}^2\, I_d,
\]
where
\begin{equation}\label{eq:delta-pp}
  \bar\delta_{\rm pp}^2
  = c_0^2
  \mathcal{R}_{\rm pp}
  (\varepsilon_{\rm mix},\rho_\star,a_\star;\eta)^2
  + \sigma_w^2\, a_\star^2.
\end{equation}
\end{corollary}

\begin{proof}
By choice of~$w$, the mean term in
Theorem~\ref{thm:value-covariance} satisfies
$\beta^w \rho \le \varepsilon_{\rm mix}\, \rho_\star$.
The fluctuation term uses
$\|\bm{\alpha}_\Dset\|_2 \le a_\star$.
Substituting into~\eqref{eq:Delta-value-cov} gives the
stated bound.
\end{proof}

\section{A Probabilistic Condition for Near-Route Alignment}
\label{app:subgaussian}

\begin{lemma}[Sub-Gaussian Short-Route Angles Imply
Near-Signal Gain]
\label{lem:subgaussian-angles}
Assume $d = 2B$. For each target-bearing near route
$j \in \Sset_i$, let the route length be $n_j = i - j$, and
suppose $n_j \le n_{\rm sig}$ for every $j \in \Sset_i$.
In block~$b$, write the step angle as
$\psi_{t,b} = \bar\psi_{t,b} + \varepsilon_{t,b}$,
and assume that, for each fixed block~$b$, the sequence
$\{\varepsilon_{t,b}\}_t$ is a martingale difference sequence
with respect to a filtration $\{\mathcal{F}_{t,b}\}_t$, and
is conditionally sub-Gaussian with variance proxy~$\sigma_\psi^2$:
\[
  \E[\varepsilon_{t,b} \mid \mathcal{F}_{t-1,b}] = 0,
  \qquad
  \E[\exp(\lambda\, \varepsilon_{t,b}) \mid
  \mathcal{F}_{t-1,b}]
  \le \exp(\lambda^2 \sigma_\psi^2 / 2)
\]
for every $\lambda \in \R$. The independent mean-zero
$\sigma_\psi^2$-sub-Gaussian case is a special case. Assume the
deterministic drift over every target-bearing near route is
bounded:
\[
  \Bigl|\sum_{t=j}^{i-1} \bar\psi_{t,b}\Bigr|
  \le \mu_{\rm sig}
\]
for every $j \in \Sset_i$ and every block~$b$.
Then, with probability at least $1 - \eta$,
\[
  \max_{j \in \Sset_i}\, \max_{1 \le b \le B}\,
  |\Theta_{j \to i, b}|
  \le \mu_{\rm sig}
  + \sigma_\psi\sqrt{2 n_{\rm sig}
  \log\frac{2 B |\Sset_i|}{\eta}},
\]
where
$\Theta_{j \to i, b} = \sum_{t=j}^{i-1} \psi_{t,b}$.
Consequently, on this event, defining
\[
  \gamma_{\rm sig}
  = \mu_{\rm sig}
  + \sigma_\psi\sqrt{2 n_{\rm sig}
  \log\frac{2 B |\Sset_i|}{\eta}},
\]
one has
$\|(P_{j \to i} - I_d)\, \Pi_U\|_\op \le \gamma_{\rm sig}$
for every $j \in \Sset_i$.
If $\gamma_{\rm sig} < 1$, then
\[
  U^\top B_{\Sset,i}^\top\, B_{\Sset,i}\, U
  \succeq (1 - \rho_\Dset)^2 (1 - \gamma_{\rm sig})^2\, U^\top U,
  \qquad \rho_\Dset = \sum_{j \in \Dset_i} \alpha_{ij}.
\]
\end{lemma}

\begin{proof}
Fix $j \in \Sset_i$ and $b \in \{1,\ldots,B\}$. Write
\[
  \Theta_{j \to i, b}
  = \sum_{t=j}^{i-1} \bar\psi_{t,b}
  + \sum_{t=j}^{i-1} \varepsilon_{t,b}.
\]
By assumption,
$|\sum_{t=j}^{i-1} \bar\psi_{t,b}| \le \mu_{\rm sig}$.
By the conditional sub-Gaussian martingale assumption, the
noise sum
$E_{j,b}=\sum_{t=j}^{i-1}\varepsilon_{t,b}$ is sub-Gaussian
with variance proxy at most $n_{\rm sig}\sigma_\psi^2$. Hence, for
every $u>0$,
\[
  \Pr(|E_{j,b}| \ge u)
  \le 2\exp\!\Bigl(-\frac{u^2}{2 n_{\rm sig}\, \sigma_\psi^2}\Bigr).
\]
Choose
$u = \sigma_\psi\sqrt{2 n_{\rm sig}
\log(2 B |\Sset_i| / \eta)}$.
A union bound over all $B|\Sset_i|$ pairs gives, with
probability at least $1-\eta$,
\[
  |\Theta_{j \to i, b}|
  \le \mu_{\rm sig}
  + \sigma_\psi\sqrt{2 n_{\rm sig}
  \log\frac{2 B |\Sset_i|}{\eta}}.
\]
For a two-dimensional rotation,
$\|R(\theta) - I_2\|_\op = 2|\sin(\theta/2)| \le |\theta|$.
Thus, by block diagonality,
$\|P_{j \to i} - I_d\|_\op \le \gamma_{\rm sig}$ and
$\|(P_{j \to i} - I_d)\Pi_U\|_\op \le \gamma_{\rm sig}$.
If $\gamma_{\rm sig} < 1$,
Lemma~\ref{lem:near-gain} applies with
$\gamma = \gamma_{\rm sig}$ and
$\rho_\Dset = \sum_{j \in \Dset_i} \alpha_{ij}$.
\end{proof}

\end{document}